\documentclass{article}

\usepackage[utf8]{inputenc} \usepackage[T1]{fontenc}    \usepackage{url}            \usepackage{booktabs}       \usepackage{amsfonts}       \usepackage{nicefrac}       \usepackage{microtype}      \usepackage{times}
\usepackage{mathtools}
\usepackage{graphicx} 

\usepackage{algorithm}
\usepackage{algorithmic}
\usepackage{multirow}
\usepackage{amsthm}
\usepackage{mathrsfs}
\usepackage{graphicx}
\usepackage{xspace}
\usepackage{float}
\usepackage{wrapfig}
\usepackage{enumitem}
\setitemize{noitemsep,topsep=0pt,parsep=2pt,partopsep=0pt}
\usepackage{natbib}

\usepackage{multibib} \newcites{sup}{Supplementary References}
\usepackage{hyperref}       \usepackage{xcolor}

\DeclareMathOperator*{\argmin}{arg\,min}
\DeclareMathOperator*{\argmax}{arg\,max}
\DeclareMathOperator{\conv}{conv}

\DeclareMathOperator{\lin}{lin}
\DeclareMathOperator{\diam}{\mathrm{diam}}
\DeclareMathOperator{\radius}{\mathrm{radius}}

\newcommand{\lmo}{\textsc{LMO}\xspace}

\DeclareMathOperator{\mdw}{mDW(\cA)}

\newcommand{\domain}{\cQ}

\usepackage[accepted]{icml2018}

\def\to{{\,\rightarrow\,}}

\mathchardef\mhyphen="2D

\providecommand{\inner}[2]{{\left\langle#1, #2\right\rangle}}

\providecommand{\brac}[1]{{\left(#1\right)}}
\newcommand{\encase}[1]{{\left[#1\right]}}

\newcommand{\norm}[1]{{ \left\lVert#1\right\rVert }}

\newcommand{\vertiii}[1]{{\left\vert\kern-0.25ex\left\vert\kern-0.25ex\left\vert #1
    \right\vert\kern-0.25ex\right\vert\kern-0.25ex\right\vert}}

\def\bb{{\mathbf{b}}}

\def\bd{{\mathbf{d}}}
\def\be{{\mathbf{e}}}

\def\bs{{\mathbf{s}}}

\def\bv{{\mathbf{v}}}

\def\bx{{\mathbf{x}}}
\def\by{{\mathbf{y}}}
\def\bz{{\mathbf{z}}}
\def\bI{{\mathbf{I}}}

\def\bM{{\mathbf{M}}}

\def\bP{{\mathbf{P}}}

\def\0{{\mathbf{0}}}

\def\bbE{{\mathbb{E}}}

\def\bbP{{\mathbb{P}}}

\def\bbR{{\mathbb{R}}}

\def\cA{\mathcal{A}}

\def\cF{\mathcal{F}}

\def\cH{\mathcal{H}}

\def\cO{\mathcal{O}}

\def\cQ{\mathcal{Q}}

\def\cZ{\mathcal{Z}}
\newtheorem*{rep@theorem}{\rep@title}
\newcommand{\newreptheorem}[2]{%
\newenvironment{rep#1}[1]{%
 \def\rep@title{#2 \ref{##1}}%
 \begin{rep@theorem}}%
 {\end{rep@theorem}}}
\newreptheorem{lemma}{Lemma'}
\newreptheorem{definition}{Definition'}
\newreptheorem{proposition}{Proposition'}
\newreptheorem{theorem}{Theorem'}

\newtheorem{theorem}{Theorem}
\newtheorem{definition}[theorem]{Definition}
\newtheorem{lemma}[theorem]{Lemma}
\renewcommand{\text}[1]{{\textnormal{#1}}}

\icmltitlerunning{On Matching Pursuit and Coordinate Descent}

\begin{document}

\twocolumn[
\icmltitle{On Matching Pursuit and Coordinate Descent}

\icmlsetsymbol{equal}{*}

\begin{icmlauthorlist}
\icmlauthor{Francesco Locatello}{equal,mpi,eth}
\icmlauthor{Anant Raj}{equal,mpi}
\icmlauthor{Sai Praneeth  Karimireddy}{epfl}\\
\icmlauthor{Gunnar R\"{a}tsch}{eth}
\icmlauthor{Bernhard  Sch\"{o}lkopf}{mpi}
\icmlauthor{Sebastian U. Stich}{epfl}
\icmlauthor{Martin Jaggi}{epfl}
\end{icmlauthorlist}

\icmlaffiliation{mpi}{Max Planck Institute for Intelligent Systems, T\"ubingen, Germany}
\icmlaffiliation{eth}{Dept. of Computer Science, ETH Zurich, Zurich, Switzerland}
\icmlaffiliation{epfl}{EPFL, Lausanne, Switzerland}

\icmlcorrespondingauthor{Francesco Locatello}{francesco.locatello@tuebingen.mpg.de}
\icmlcorrespondingauthor{Anant Raj}{anant.raj@tuebingen.mpg.de}

\icmlkeywords{Machine Learning, ICML}

\vskip 0.3in
]

\printAffiliationsAndNotice{\icmlEqualContribution}

\begin{abstract}
Two popular examples of first-order optimization methods over linear spaces are coordinate descent and matching pursuit algorithms, with their randomized variants. While the former targets the optimization by moving along coordinates, the latter considers a generalized notion of directions. Exploiting the connection between the two algorithms, we present a unified analysis of both, providing affine invariant sublinear $\cO(1/t)$ rates on smooth objectives and linear convergence on strongly convex objectives. As a byproduct of our affine invariant analysis of matching pursuit, our rates for steepest coordinate descent are the tightest known. Furthermore, we show the first accelerated convergence rate $\cO(1/t^2)$ for matching pursuit and steepest coordinate descent on convex objectives.
\end{abstract}

\section{Introduction}\label{sec:intro}
In this paper we address the following convex optimization problem:
\begin{equation}\label{eq:MPproblem}
\min_{\bx\in\lin(\cA)} f(\bx) \,,
\end{equation}
where $f$ is a convex function. The minimization is over a linear space, which is parametrized as the set of linear combinations of elements from a given set $\cA$.  These elements of $\cA$ are called \textit{atoms}. In the most general setting, $\cA$ is assumed to be a compact but not necessarily finite subset of a Hilbert space, i.e., a linear space equipped with an inner product, complete in the corresponding norm. Problems of the form~\eqref{eq:MPproblem} are tackled by a multitude of first-order optimization methods and are of paramount interest in the machine learning community~\cite{seber2012linear, meir2003introduction, scholkopf2001learning,menard2018applied,tibshirani2015general}.  

Traditionally, matching pursuit (MP) algorithms were introduced to solve the inverse problem of representing a measured signal by a sparse combination of atoms from an over-complete basis~\cite{Mallat:1993gu}. In other words, the solution of the optimization problem \eqref{eq:MPproblem} is formed as a linear combination of few of the elements of the atom set~$\cA$ -- i.e. a sparse approximation.
At each iteration, the MP algorithm picks a direction from $\cA$ according to the gradient information, and takes a step. This procedure is not limited to atoms of fixed dimension. Indeed, $\lin(\cA)$ can be an arbitrary linear subspace of the ambient space and we are interested in finding the minimizer of $f$ only on this domain, see e.g.~\cite{gillis2018fast}. Conceptually, MP stands in the middle between coordinate descent (CD) and gradient descent, as the algorithm is allowed to descend the function along a prescribed set of directions which does not necessarily correspond to coordinates.
This is particularly important for machine learning applications as it translates to a sparse representation of the iterates in terms of the elements of~$\cA$ while maintaining the convergence guarantees~\cite{LacosteJulien:2013ue,locatello2017greedy}.

The first analysis of the MP algorithm in the optimization sense to solve the template~\eqref{eq:MPproblem} without incoherence assumptions was done by~\cite{locatello2017unified}. To prove convergence, they exploit the connection between MP and the Frank-Wolfe (FW) algorithm~\cite{Frank:1956vp}, a popular projection-free algorithm for the constrained optimization case. On the other hand, steepest coordinate descent is a special case of MP (when the atom set is the L1 ball). This is particularly important as the CD rates can be deduced from the MP rates. Furthermore, the literature on coordinate descent is currently much richer than the one on MP. Therefore, understanding the connection of the two classes of CD and MP-type algorithms is a main goal of this paper, and results in benefits for both sides of the spectrum.
In particular, the contributions of this paper are:

\begin{itemize}
\item We present an affine invariant convergence analysis for Matching Pursuit algorithms solving~\eqref{eq:MPproblem}. Our approach is tightly related to the analysis of coordinate descent and relies on the properties of the atomic norm in order to generalize from coordinates to atoms.
\item Using our analysis, we present the tightest known linear and sublinear convergence rates for steepest coordinate descent, improving the constants in the rates of~\cite{stich2017approximate,Nutini:2015vd}.
\item We discuss the convergence guarantees of Random Pursuit (RP) methods which we analyze through the lens of MP. In particular, we present a unified analysis of both MP and RP which allows us to carefully trade off the use of (approximate) steepest directions over random ones.
\item We prove the first known accelerated rate for MP, as well as for steepest coordinate descent. As a consequence, we also demonstrate an improvement upon the accelerated random CD rate by performing a steepest coordinate update instead.
\end{itemize}
\paragraph{Related Work:}\label{par:related}
Matching Pursuit was introduced in the context of sparse recovery~\cite{Mallat:1993gu}, and later, fully corrective variants similar to the one used in Frank-Wolfe~\cite{Holloway:1974ii,LacosteJulien:2015wj,thomas2018subsampling} were introduced under the name of orthogonal matching pursuit \cite{chen1989orthogonal, Tropp:2004gc}. The classical literature for MP-type methods is typically focused on recovery guarantees for sparse signals and the convergence depends on very strong assumptions (from an optimization perspective), such as incoherence or restricted isometry properties of the atom set \cite{Tropp:2004gc,davenport2010analysis}. Convergence rates with incoherent atom sets are predented in \cite{Gribonval:2006ch,Temlyakov:2013wf, Temlyakov:2014eb, nguyen2014greedy}. Also boosting can be seen as a generalized coordinate descent method over a hypothesis class~\cite{ratsch2001convergence,meir2003introduction}.

The idea of following a prescribed set of directions also appears in the field of derivative free methods. For instance, the early method of Pattern-Search~\cite{Hooke:1961,Dennis:1991,torczon1997convergence} explores the search space by probing function values along predescribed directions (``patterns'' or atoms). This method is in some sense orthogonal to the approach here: by probing the function values along all atoms, one aims to find a direction along which the function decreases (and the absolute value of the scalar product with the gradient is potentially small). MP does not access the function value, but computes the gradient and then picks the atom with the smallest scalar product with the gradient, and then moves to a point where the function value decreases.

The description of random pursuit appears already in the work of~\citet{Rastrigin:1964} and was first analyzed by~\citet{Karmanov:1974b,Karmanov:1974a,Zielinski:1983}. More recently random pursuit was revisited in~\cite{stich2013optimization,Stich14}.

Acceleration of first-order methods was first developed in~\cite{Nesterov:1983wy}.
An accelerated CD method was described in~\cite{nesterov2012efficiency}.
The method was extended in~\cite{Lee13} for non-uniform sampling, and later in~\cite{Stich14} for optimization along arbitrary random directions. Recently, optimal rates have been obtained for accelerated CD~\cite{Nesterov:2017,zhuc16}. A close setup is the accelerated algorithm presented in~\cite{el2017general}, which minimizes a composite problem of a convex function on $\bbR^n$ with a non-smooth regularizer which acts as prior for the structure of the space. Contrary to our setting, the approach is restricted to the atoms being linearly independent.
Simultaneously at ICML 2018, \citet{lu2018greedy} propose an accelerated rate for the semi-greedy coordinate descent which is a special case of our accelerated MP algorithm.
\vspace{-3mm} 
\paragraph{Notation:}
Given a non-empty subset $\cA$ of some Hilbert space, let $\conv(\cA)$ be the convex hull of~$\cA$, and let $\lin(\cA)$ denote its linear span. Given a closed 
set $\cA$, we call its diameter $\diam(\cA)=\max_{\bz_1,\bz_2\in\cA}\|\bz_1-\bz_2\|$ and its radius $\radius(\cA) = \max_{\bz\in\cA}\|\bz\|$. 
 $\|\bx\|_\cA := \inf \{ c > 0 \colon \bx \in c \cdot \conv (\cA) \}$ is the atomic norm of $\bx$ over a set $\cA$ (also known as the gauge function of $\conv (\cA)$). We call a subset $\cA$ of a Hilbert space symmetric if it is closed under negation. 
\section{Revisiting Matching Pursuit} \label{sec:match_pursuit_revist}
Let $\cH$ be a Hilbert space with associated inner product $\langle \bx, \by\rangle, \,\forall \, \bx,\by \in \cH$. The inner product induces the norm $\| \bx \|^2 := \langle \bx, \bx \rangle,$ $\forall \, \bx \in \cH$. Let $\cA \subset \cH$ be a compact and symmetric set (the ``set of atoms'' or dictionary) and let $f \colon \cH \to \bbR$ be convex and $L$-smooth ($L$-Lipschitz gradient in the finite dimensional case). If $\cH$ is an infinite-dimensional Hilbert space, then $f$ is assumed to be Fr{\'e}chet differentiable. 
   \begin{algorithm}
\caption{Generalized Matching Pursuit}
  \label{algo:MP}
\begin{algorithmic}[1]
  \STATE \textbf{init} $\bx_{0} \in \lin(\cA)$
  \STATE \textbf{for} {$t=0\dots T$}
  \STATE \quad Find $\bz_t := (\text{Approx-}) \lmo_\cA(\nabla f(\bx_{t}))$
  \STATE \quad $\bx_{t+1} := \bx_t - \frac{\langle \nabla f(\bx_t),\bz_t\rangle}{L\|\bz_t\|^2}\bz_t$
  \STATE \textbf{end for}
\end{algorithmic}
\end{algorithm}
\vspace{-0.2mm}

In each iteration, MP queries a linear minimization oracle (\lmo) to find the steepest descent direction among the set~$\cA$:
\begin{equation}\label{eq:FWLMO}
\lmo_\cA(\by) := \argmin_{\bz\in\cA} \,\langle \by, \bz \rangle \,,
\end{equation}
for a given query vector $\by\in\cH$.
This key subroutine is shared with the Frank-Wolfe method~\cite{Frank:1956vp,jaggi13FW} as well as steepest coordinate descent. Indeed, finding the steepest coordinate is equivalent to minimizing Equation~\ref{eq:FWLMO}.  
The MP update step minimizes a quadratic upper bound $
 g_{\bx_{t}}(\bx) = f(\bx_{t}) + \langle\nabla f(\bx_{t}), \bx-\bx_{t}\rangle+\frac{L}{2}\|\bx-\bx_{t}\|^2
$
of~$f$ at~$\bx_t$ on the direction $\bz$ returned by the \lmo, where $L$ is an upper bound on the smoothness constant of~$f$ with respect to the Hilbert norm $\|\cdot\|$.
For $f(\bx) = \frac{1}{2}\|\by -\bx\|^2$, $\by \in \cH$, Algorithm~\ref{algo:MP} recovers the classical MP algorithm \cite{Mallat:1993gu}.
\vspace{-2mm} 
\paragraph{The LMO.}
Greedy and projection-free optimization algorithms such as Frank-Wolfe and Matching Pursuit rely on the property that the result of the $\lmo$ is a descent direction, which is translated to an \emph{alignment assumption} of the search direction returned by the \lmo (i.e., $\bz_t$ in Algorithm~\ref{algo:MP}) and the gradient of the objective at the current iteration (see~\cite{locatello2017greedy},~\citep[third premise]{pena2015polytope} and~\citep[Lemma 12 and proof of Proposition 6.4]{torczon1997convergence}). Specifically, for Algorithm~\ref{algo:MP}, a symmetric atom set $\cA$ ensures that $\langle \nabla f(\bx_t), \bz_t \rangle < 0$, as long as $\bx_t$ is not optimal yet. Indeed, we then have that $\min_{\bz\in\cA}\langle \nabla f(\bx_t), \bz \rangle =\min_{\bz\in\conv(\cA)}\langle \nabla f(\bx_t), \bz \rangle < 0$ where the inequality comes from symmetry as $\bz = \0 \in \conv(\cA)$. Note that an alternative sufficient condition instead of symmetry is that $\cA$ is the atomic ball of a norm (the so called atomic norm~\cite{Chandrasekaran:2012hl}).
\vspace{-2mm} 
\paragraph{Steepest Coordinate Descent.}
In the case when $\cA$ is the L1-ball, the MP algorithm becomes identical to  steepest coordinate descent \cite{nesterov2012efficiency}. Indeed, due to symmetry of $\cA$, one can rewrite the $\lmo$ problem as 
$i_t = \argmax_{i} | \nabla_i f(x)| \,,$ where $\nabla_i$ is the $i$-th component of the gradient, i.e. $\langle \nabla f(x), \be_i\rangle$ with $\be_i$ being one of the natural vectors. Then the update step can be written as:
\begin{align*}
\bx_{t+1} := \bx_{t+1} - \frac{1}{L}\nabla_{i_t} f(\bx_t)\be_i \,.
\end{align*}
Note that by assuming a symmetric atom set and solving the $\lmo$ problem as defined in~\eqref{eq:FWLMO} the steepest atom is aligned with the negative gradient, therefore the positive stepsize $- \frac{\langle \nabla f(\bx_t), \bz_t\rangle}{L}$ decreases the objective.
\paragraph{Approximate linear oracles.}
\label{sec:approxlmo}
Exactly solving the \lmo defined in \eqref{eq:FWLMO} can be costly in practice, both in the MP and the CD setting, as $\cA$ can contain (infinitely) many atoms. On the other hand, approximate versions can be much more efficient. Algorithm~\ref{algo:MP} allows for an \emph{approximate \lmo}. Different notions of such a \lmo were explored for MP and OMP in \cite{Mallat:1993gu} and \cite{Tropp:2004gc}, respectively, for the Frank-Wolfe framework in~\cite{jaggi13FW,LacosteJulien:2013ue} and for coordinate descent~\cite{stich2017approximate}.
For given quality parameter $\delta\in \left( 0,1\right]$ and given direction $\bd\in\cH$, the approximate \lmo for Algorithm~\ref{algo:MP} returns a vector $\tilde{\bz}\in\cA$ such that: \begin{equation}\label{eq:inexactLMOMP}
 \langle \bd,\tilde{\bz}\rangle \leq \delta\langle \bd,\bz\rangle \,,
\end{equation}
relative to $\bz =\lmo_\cA(\bd)$ being an exact solution.
\subsection{Affine Invariant Algorithm}
In this section, we will present our new affine invariant algorithm for the optimization problem~\eqref{eq:MPproblem}. Hence, we first explain in Definition~\ref{def:aff_invar} that what does it mean for an optimization algorithm to be affine invariant:
\begin{definition} \label{def:aff_invar}
An optimization method is called \emph{affine invariant} if it is invariant
under affine transformations of the input problem: If one chooses any
re-parameterization of the domain~$\domain$ by a \emph{surjective} linear or
affine map $\bM:\hat\domain\rightarrow\domain$, then the ``old'' and ``new''
optimization problems $\min_{\bx\in\domain}f(\bx)$ and
$\min_{\hat\bx\in\hat\domain}\hat f(\hat\bx)$ for $\hat f(\hat\bx):=f(\bM\hat\bx)$
look the same to the algorithm.
\end{definition}
In other words, a step of the algorithm in the original optimization problem is the same as a step in the transformed problem. We will further demonstrate in the appendix that the proposed Algorithm~\ref{algo:affineMP} which we discuss later in detail is indeed an affine invariant algorithm.  In order to obtain an affine invariant algorithm, we define an affine invariant notion of smoothness using the atomic norm.  This notion is inspired by the curvature constant employed in FW and MP, see~\cite{jaggi13FW,locatello2017unified}.
We define:\vspace{-1mm}
\begin{equation}
 L_{\cA} := \!\!\!\sup_{\substack{\bx,\by\in\lin(\cA)\\\by = \bx + \gamma\bz \\\|\bz\|_\cA=1, \gamma\in \bbR_{>0}}}\frac{2}{\gamma^2}\big[ f(\by)- f(\bx) -  \langle \nabla f(\bx), \by - \bx \rangle\big] \,.
\end{equation}
This definition combines the complexity of the function~$f$ as well as the set $\cA$ into a single number, and is affine invariant under transformations of our input problem~\eqref{eq:MPproblem}. It yields the same upper bound to the function as the one given by the traditional smoothness definition, that is $L_{\cA}$-smoothness with respect to the atomic norm $\|\cdot\|_\cA$, when~$\bx,\by$ are constrained to the set $\lin(\cA)$:\vspace{-1mm}
\begin{align*}
f(\by) \leq f(\bx) + \langle \nabla f(\bx),\by-\bx \rangle + \frac{L_\cA}{2} \|\by-\bx\|_\cA \,,
\end{align*}
For example, if $\cA$ is the L1-ball we obtain $f(\bx + \gamma\bz) \leq f(\bx) + \gamma  \langle \nabla f(\bx),\bz \rangle + \gamma^2\frac{L_1}{2}$ where $\|\bz\|_1 = 1$. Based on the affine-invariant notion of smoothness defined above, we now present pseudocode of our affine-invariant method in Algorithm~\ref{algo:affineMP}.

\begin{algorithm}[H]
\caption{Affine Invariant Generalized Matching Pursuit}
  \label{algo:affineMP}
\begin{algorithmic}[1]
  \STATE \textbf{init} $\bx_{0} \in \lin(\cA)$
  \STATE \textbf{for} {$t=0\dots T$}
  \STATE \quad Find $\bz_t := (\text{Approx-}) \lmo_\cA(\nabla f(\bx_{t}))$
  \STATE \quad $\bx_{t+1} = \bx_t - \frac{\langle \nabla f(\bx_t),\bz_t\rangle}{L_\cA}\bz_t$
  \STATE \textbf{end for}
\end{algorithmic}
\end{algorithm}
\vspace{-0.2mm}

The above algorithm looks very similar to the generalized MP (Algorithm~\ref{algo:MP}), however, the main difference is that while the original algorithm is not affine invariant over the domain $\domain=\lin(\cA)$ (Def~\ref{def:aff_invar}), the new Algorithm \ref{algo:affineMP} is so, due to using the generalized smoothness constant $L_\cA$.

\paragraph{Note.}
\looseness=-1For the purpose of the analysis, we call $\bx^\star$ the minimizer of problem~\eqref{eq:MPproblem}. If the optimum is not unique, we pick the one with largest atomic norm as it represent the worst case for the analysis. All the proofs are deferred to the appendix.
\subsubsection{New Affine Invariant Sublinear Rate}\label{sec:sub_rate}
In this section, we will provide the theoretical justification of our proposed approach for smooth functions (sublinear rate) and its theoretical comparison with existing previous analysis for special cases. We define the level set radius measured with the atomic norm as:
\begin{equation}
R_\cA^2 := \max_{\substack{\bx\in\lin(\cA)\\ f(\bx)\leq f(\bx_0)}}\|\bx-\bx^\star\|^2_\cA \,.
\end{equation}
When we measure this radius with the $\|\cdot\|_2$ we call it $R_2^2$, and when we measure it with $\|\cdot\|_1$ we call it $R_1^2$.
Note that measuring smoothness using the atomic norm guarantees that for the Lipschitz constant $L_\cA$ the following holds:
\begin{lemma}\label{lem:LwithRadius}
Assume $f$ is $L$-smooth w.r.t. a given norm $\|\cdot\|$, over $\lin(\cA)$ where $\cA$ is symmetric.
Then,
\begin{equation}
 L_\cA \leq L \, \radius_{\norm{\cdot}}(\cA)^2\,.
\end{equation}
\end{lemma}
For example, in the coordinate descent setting we measure smoothness with the atomic norm being the L1-norm. Lemma~\ref{lem:LwithRadius} implies that  $L_{\cA} \leq L_1 \leq L_2$ where $L_2$ is the smoothness constant measured with the L2-norm. Note that the radius of the L1-ball measured with $\|\cdot\|_1$ is 1.
Therefore, we put ourselves in a more general setting than Algorithm~\ref{algo:MP}, showing convergence of the affine invariant Algorithm~\ref{algo:affineMP}

We are now ready to prove the convergence rate of Algorithm~\ref{algo:affineMP} for smooth functions.
\begin{theorem}\label{thm:sublinear_MP_rate}
Let $\cA \subset \cH$ be a closed and bounded set. We assume that $\|\cdot\|_\cA$ is a norm over $\lin(\cA)$. Let $f$ be convex and $L_\cA$-smooth w.r.t. the norm $\|\cdot\|_\cA$ over $\lin(\cA)$, and let $R_\cA$ be the radius of the level set of $\bx_0$ measured with the atomic norm.
Then, Algorithm~\ref{algo:affineMP} converges for $t \geq 0$ as
\[
f(\bx_{t+1}) - f(\bx^\star) \leq \frac{2L_{\cA}R_\cA^2}{\delta^2(t+2)} \,,
\]
where $\delta \in (0,1]$ is the relative accuracy parameter of the employed approximate \lmo~\eqref{eq:inexactLMOMP}.
\end{theorem}

\paragraph{Discussion.}
The proof of Theorem~\ref{thm:sublinear_MP_rate} extends the convergence analysis of steepest coordinate descent. As opposed to the classical proof in~\cite{nesterov2012efficiency}, the atoms are here not orthogonal to each other, do not have the same norm and do not correspond to the coordinates of the ambient space. Indeed, $\lin(\cA)$ could be a subset of the ambient space and the only assumptions on $\cA$ are that is closed, bounded and $\|\cdot\|_\cA$ is a norm over $\lin(\cA)$. We do not make any incoherence assumption. The key element of our proof is the definition of smoothness using the atomic norm. Furthermore, we use the properties of the atomic norm to obtain a proof which shares the spirit of the Nesterov's one without having to rely on strong assumptions on $\cA$.
\vspace{-2mm}
\paragraph{Relation to Previous MP Sublinear Rate.}
The sublinear convergence rate presented in Theorem~\ref{thm:sublinear_MP_rate} is fundamentally different in spirit from the one proved in~\cite{locatello2017unified}. Indeed, their convergence analysis builds on top of the proof technique used for Frank-Wolfe in~\cite{jaggi13FW}. They introduce a dependency from the atomic norm of the iterates as a way to constrain the part of the space in which the optimization is taking place which artificially induce a notion of duality gap. They do so by defining $\rho := \max\left\lbrace \|\bx^\star\|_{\cA}, \|\bx_{0}\|_{\cA}\ldots,\|\bx_T\|_{\cA}\right\rbrace<\infty$.
\cite{locatello2017unified} also used an affine invariant notion of smoothness, thus obtaining an affine invariant rate. On the other hand, their notion of smoothness depends explicitly on $\rho$. While this constant can be further upper bounded with the level set radius, it is not known a priori, which makes the estimation of the smoothness constant problematic as it is needed in the algorithm and the proof technique more involved.
We propose a much more elegant solution, which uses a different affine invariant definition of smoothness which explicitly depend on the atomic norm. Furthermore, we managed to get rid of the dependency on the sequence of the iterates by using only properties of the atomic norm without any additional assumption (finiteness of $\rho$).
\vspace{-2mm}
\paragraph{Relation to Steepest Coordinate Descent.}
From our analysis, we can readily recover existing rates for coordinate descent. Indeed, if $\cA$ is the L1-ball in an $n$ dimensional space, the rate of Theorem~\ref{thm:sublinear_MP_rate} with exact oracle can be written as:
\begin{align*}
f(\bx_{t+1}) - f(\bx^\star) \leq \frac{2L_1R_1^2}{t+2}\leq \frac{2L_2R_1^2}{t+2} \leq \frac{2L_2nR_2^2}{t+2} \,,
\end{align*}
where the first inequality is our rate, the second inequality is the rate of~\cite{stich2017approximate} and the last inequality is the rate given in \cite{nesterov2012efficiency}, both with global Lipschitz constant. Therefore, by measuring smoothness with the atomic norm, we have shown a tighter dependency on the dimensionality of the space. Indeed, the atomic norm gives the tightest norm to measure the product between the smoothness of the function and the level set radius among the known rates. 
Therefore, our rate for steepest coordinate descent is the tightest known\footnote{Note that for coordinate-wise $L$ our definition is equivalent to the classical one%, except our is still affine invariant
. $L_\cA\leq L_2$ if the norm is defined over more than one dimension (i.e. blocks), otherwise there is equality. For the relationship of $L_1$-smoothness to coordinate-wise smoothness, see also \citep[Theorem 4 in Appendix]{karimireddy2019efficient}.}.

\paragraph{Coordinate Descent and Affine Transformations.}
\looseness=-1
But what does it mean to have an affine invariant rate for coordinate descent? By definition, it means that if one applies an affine transformation to the L1-ball, the coordinate descent algorithm in the natural basis and on the transformed domain~$\hat{Q}$ are equivalent. Note that in the transformed problem, the coordinates do not corresponds to the natural coordinates anymore. Indeed, in the transformed domain the coordinates are $\hat{\be_i} = \bM^{-1} \be_i$ where $\bM^{-1}$ is the inverse of the affine map $\bM: \hat{Q}\rightarrow Q$. If one would instead perform coordinate descent in the transformed space using the natural coordinates, one would obtain not only different atoms but also a different iterate sequence. In other words, while Matching Pursuit is fully affine invariant, the definition of CD is not, as the choice of the coordinates is not part of the definition of the optimization problem. The two algorithms do coincide for one particular choice of basis, the canonical coordinate basis for $\cA$. 
\subsubsection{Sublinear Rate of Random Pursuit}
\looseness=-1There is a significant literature on optimization methods which do not require full gradient information. A notable example is random coordinate descent, where only a random component of the gradient is known.
As long as the direction that is selected by the \lmo is not orthogonal to the gradient we have convergence guarantees due to the inexact oracle definition.
We now abstract from the random coordinate descent setting and analyze a randomized variant of matching pursuit, the \emph{random pursuit} algorithm, in which the atom $\bz$ is randomly sampled from a distribution over $\cA$, rather than picked by a linear minimization oracle. This approach is particularly interesting, as it is deeply connected to the random pursuit algorithm analyzed in~\cite{stich2013optimization}.
For now we assume that we can compute the projection of the gradient onto a single atom $\langle \nabla f, \bz\rangle$ efficiently. In order to present a general recipe for any atom set, we exploit the notion of inexact oracle and define the inexactness of the expectation of the sampled direction for a given sampling distribution:\vspace{-1mm}
\begin{equation}\label{eq:delta-def}
\hat\delta^2 := \min_{\bd\in\lin(\cA)}\frac{\bbE_{\bz \in \cA} \langle \bd,\bz\rangle^2}{\|\bd\|_{\cA*}^2} \,.
\end{equation}
This constant was already used in~\cite{Stich14} to measure the convergence of random pursuit ($\beta^2$ in his notation).
Note that for uniform sampling from the corners of the L1-ball, we have $\hat\delta^2 = \frac{1}{n}$. Indeed, $\bbE_{\bz \in \cA} \langle \bd,\bz\rangle^2 = \frac{1}{n}$ for any $\bd$. 
This definition holds for any sampling scheme as long as $\hat\delta^2\neq 0$. Note that by using this quantity we do not get the tightest possible rate, as at each iteration, we consider how much worse a random update could be compared to the optimal (steepest) update. 

We are now ready to present the sublinear convergence rate of random matching pursuit.
\begin{theorem}\label{thm:rand_purs_sub}
Let $\cA \subset \cH$ be a closed and bounded set. We assume that $\|\cdot\|_\cA$ is a norm. Let $f$ be convex and $L_\cA$-smooth w.r.t. the norm $\|\cdot\|_\cA$ over $\lin(\cA)$ and let $R_\cA$ be the radius of the level set of $\bx_0$ measured with the atomic norm.
Then, Algorithm~\ref{algo:affineMP} converges for $t \geq 0$ as\vspace{-1mm}
\[
\bbE_\bz \big[f(\bx_{t+1}) \big] - f(\bx^\star) \leq \frac{2L_{\cA}R_\cA^2}{\hat\delta^2(t+2)} \,,
\]
when the $\lmo$ is replaced with random sampling of $\bz$ from a distribution over $\cA$.
\end{theorem}

\paragraph{Gradient-Free Variant.}
If is possible to obtain a fully gradient-free optimization scheme. In addition to having replaced the LMO in Algorithm~\ref{algo:MP} by the random sampling as above, as can additionally also replace the line search step on the quadratic upper bound given by smoothness, with instead an approximate line search on $f$. As long as the update scheme guarantees as much decrease as the above algorithm, the convergence rate of Theorem~\ref{thm:rand_purs_sub} holds. 

\paragraph{Discussion.}
This approach is very general, as it allows to guarantee convergence for \textit{any} sampling scheme and \textit{any} set $\cA$ provided that $\hat\delta^2 \neq 0$. In the coordinate descent case we have that for the worst possible gradient for random has $\hat\delta^2 = \frac{1}{n}$. Therefore, the speed-up of steepest can be up to a factor equal to the number of dimensions in the best case. Similarly, if $\bz$ is sampled from a spherical distribution, $\hat\delta^2 = \frac{1}{n}$~\cite{stich2013optimization}. More examples of computation of $\hat\delta^2$ can be found in \citep[Section 4.2]{Stich14}. Last but not least, note that $\hat\delta^2$ is affine invariant as long as the sampling distribution over the atoms is preserved. 

\subsubsection{Strong Convexity and Affine Invariant Linear Rates}
Similar to the affine invariant notion of smoothness, we here define the affine invariant notion of strong convexity.
\begin{align*}
\mu_\cA := \inf_{\substack{\bx,\by\in\lin(\cA)\\\bx\neq\by}} \frac{2}{\|\by-\bx\|_{\cA}^2} D(\by,\bx)\,.
\end{align*}
where $D(\by,\bx):= f(\by)-f(\bx)-\langle \nabla f(\bx),\by -\bx\rangle$
We can now show the linear convergence rate of both the matching pursuit algorithm and its random pursuit variant.
\begin{theorem}\label{thm:linear_rate}
Let $\cA \subset \cH$ be a closed and bounded set. We assume that $\|\cdot\|_\cA$ is a norm. Let $f$ be $\mu_\cA$-strongly convex and $L_\cA$-smooth w.r.t. the norm $\|\cdot\|_\cA$, both over $\lin(\cA)$.
Then, Algorithm~\ref{algo:affineMP} converges for $t \geq 0$ as
\[
\varepsilon_{t+1} \leq \big(1 - \delta^2\frac{\mu_\cA}{L_{\cA}}\big)\varepsilon_t \,.
\]
where $\varepsilon_{t} := f(\bx_{t}) - f(\bx^\star)$.
If the LMO direction is sampled randomly from $\cA$,  Algorithm~\ref{algo:affineMP} converges for $t \geq 0$ as
\[
\bbE_\bz \left[\varepsilon_{t+1}|\bx_t\right] \leq \big(1 - \hat\delta^2\frac{\mu_\cA}{L_{\cA}} \big)\varepsilon_t \,.
\]
\end{theorem}

\paragraph{Relation to Previous MP Linear Rate.}
Again, the proof of Theorem~\ref{thm:linear_rate} extends the convergence analysis of steepest coordinate descent using solely the affine invariant definition of strong convexity and the properties of the atomic norm.
Note that again we define the strong convexity constant without relying on $\rho = \max\left\lbrace \|\bx^\star\|_{\cA}, \|\bx_{0}\|_{\cA}\ldots,\|\bx_T\|_{\cA}\right\rbrace<\infty$ as in~\cite{locatello2017unified}. We now show that our choice of the strong convexity parameter is the tightest w.r.t. any choice of the norm and that we can precisely recover the non affine invariant rate of~\cite{locatello2017unified}.
Let us recall their notion of \textit{minimal directional width}, which is the crucial constant to measure the geometry of the atom set for a fixed norm:\vspace{-1mm}
\begin{align*}
\mdw := \min_{\substack{\bd\in\lin(\cA)\\\bd\neq 0}} \max_{z\in\cA} \Big\langle \frac{\bd}{\|\bd\|},\bz \Big\rangle \,.
\end{align*}
Note that for CD we have that $\mdw = \frac{1}{\sqrt{n}}$.
Now, we relate the affine invariant notion of strong convexity with the minimal directional width and the strong convexity w.r.t. any chosen norm. This is important, as we want to make sure to perfectly recover the convergence rate given in~\cite{locatello2017unified}.
\begin{lemma}\label{thm:mu_mdw}
Assume $f$ is $\mu$-strongly convex w.r.t. a given norm $\|\cdot\|$ over $\lin(\cA)$ and $\cA$ is symmetric. Then:\vspace{-1mm}
\begin{align*}
\mu_\cA \geq \mdw^2\mu \,.
\end{align*}
\end{lemma}
\noindent We then recover their non-affine-invariant rate as:
\begin{align*}
\varepsilon_{t+1}\leq \Big(1 - \delta^2\frac{\mu\mdw^2}{L\radius_{\|\cdot\|}(\cA)^2} \Big) \varepsilon_t \,.
\end{align*}

\paragraph{Relation to Coordinate Descent.}
When we fix $\cA$ as the L1-ball and use an exact oracle our rate becomes:
\begin{align*}
\varepsilon_{t+1}\leq\left(1 - \frac{\mu_1}{L_1}\right)\varepsilon_t \leq\left(1 - \frac{\mu_1}{L}\right)\varepsilon_t \leq \left(1 - \frac{\mu}{nL}\right)\varepsilon_t \,,
\end{align*}
where the first is our rate, the second is the rate of steepest CD~\cite{Nutini:2015vd} and the last is the one for randomized CD~\cite{nesterov2012efficiency} ($n$ is the dimension of the ambient space).
Therefore, our linear rate for coordinate descent is the tightest known.

\vspace{-2mm}
\section{Accelerating Generalized Matching Pursuit} \label{sec:acc_match_pursuit}
\looseness=-1
As we established in the previous sections, matching pursuit can be considered a generalized greedy coordinate descent where the allowed directions do not need to form an orthogonal basis. This insight allows us to generalize the analysis of accelerated coordinate descent methods and to accelerate matching pursuit \cite{Lee13,Nesterov:2017}. However it is not clear at the outset how to even accelerate greedy coordinate descent, let alone the matching pursuit method. Recently \citet{song2017accelerated} proposed an accelerated greedy coordinate descent method by using the linear coupling framework of \cite{AllenZhu:2014uv}. However the updates they perform at each iteration are not guaranteed to be sparse which is critical for our application. We instead extend the acceleration technique in \cite{stich2013optimization} which in turn is based on \cite{Lee13}. They allow the updates to the two sequences of iterates $\bx$ and $\bv$ to be chosen from any distribution. If this distribution is chosen to be over coordinate directions, we get the familiar accelerated coordinate descent, and if we instead chose the distribution to be over the set of atoms, we would get an accelerated random pursuit algorithm. To obtain an accelerated \emph{matching} pursuit algorithm, we need to additionally \emph{decouple} the updates for $\bx$ and $\bv$ and allow them to be chosen from different distributions. We will update $\bx$ using the greedy coordinate update (or the matching pursuit update), and use a random coordinate (or atom) direction to update~$\bv$.

The possibility of decoupling the updates was noted in \citep[Corollary 6.4]{Stich14} though its implications for accelerating greedy coordinate descent or matching pursuit were not explored.
From here on out, we shall assume that the linear space spanned by the atoms $\cA$ is finite dimensional. This was not necessary for the non-accelerated matching pursuit and it remains open if it is necessary for accelerated MP. When sampling, we consider only a non-symmetric version of the set $\cA$ with all the atoms in the same half space. Line search ensures that sampling either $\bz$ or $-\bz$ yields the same update. For simplicity, we focus on an exact \lmo.

\subsection{From Coordinates to Atoms}
For the acceleration of MP we make some stronger assumption w.r.t. the rates in the previous section. In particular, we will not obtain an affine invariant rate which remains an open problem. The key challenges for an affine invariant accelerated rate are strong convexity of the model, which can be solved using arguments similar to~\cite{d2013optimal} and the fact that our proof relies on defining a new norm which deform the space in order to obtain favorable sampling properties as we will explain in this section.
The main difference between working with atoms and working with coordinates is that projection along coordinate basis vectors is 'unbiased'. Let $\be_i$ represent the $i$th coordinate basis vector. Then for some vector $\bd$, if we project along a random basis vector $\be_i$,\vspace{-1mm}
$$
  \bbE_{i \in [n]} [\inner{\be_i}{\bd}\be_i ]= \frac{1}{n}\bd\,.
$$
However if instead of coordinate basis, we choose from a set of atoms $\cA$, then this is no longer true. We can correct for this by morphing the geometry of the space. Suppose we sample the atoms from a distribution $\cZ$ defined over $\cA$. Let us define
$$
  \tilde{\bP} := \bbE_{\bz \sim \cZ}[\bz \bz^\top]\,.
$$

We assume that the distribution $\cZ$ is such that $\lin(\cA) \subseteq \text{range}(\tilde{\bP})$. This intuitively corresponds to assuming that there is a non-zero probability that the sampled $\bz \sim \cZ$ is along the direction of every atom $\bz_t \in \cA$ i.e.
$$
  \bbP_{\bz \sim \cZ}[\inner{\bz}{\bz_t} >0] > 0, \ \forall \bz_t \in \cA\,.
$$
 Further let $\bP = \tilde{\bP}^{\dagger}$ be the pseudo-inverse of $\tilde{\bP}$. Note that both $\bP$ and $\tilde{\bP}$ are positive semi-definite matrices. We can equip our space with a new inner product
$\inner{\cdot}{\bP \cdot}$ and the resulting norm $\norm{\cdot}_\bP$. With this new dot product,
$$
  \bbE_{\bz \sim \cZ} [\inner{\bz}{\bP \bd}\bz] = \bbE_{\bz \sim \cZ} [ \bz \bz^\top]\bP \bd = \bP^{\dagger} \bP \bd = \bd\,.
$$
The last equality follows from our assumption that $\lin(\cA) \subseteq \text{range}(\tilde{\bP})$.

\begin{algorithm}[h]
\begin{algorithmic}[1]
  \STATE \textbf{init} $\bx_0 = \bv_0 = \by_0$, $\beta_0 = 0$, and $\nu'$
  \STATE \textbf{for} {$t=0, 1 \dots T$}
       \STATE \qquad  Solve ${\alpha_{t+1}^2}{L \nu'} = \beta_t + \alpha_{t+1}$   
    \STATE  \qquad $\beta_{t+1} := \beta_t + \alpha_{t+1} $
    \STATE \qquad $\tau_t := \frac{\alpha_{t+1}}{\beta_{t+1}}$
   \STATE \qquad Compute $\by_{t} := (1- \tau_t)\bx_{t} + \tau_t \bv_t$
   \STATE \qquad Sample $\bz_t \sim \cZ$
\STATE \qquad $\bx_{t+1} := \by_t - \frac{\langle \nabla f(\by_t),\bz_t\rangle}{L\|\bz_t\|^2_2}\bz_t$
\STATE \qquad $\bv_{t+1} := \bv_t - \alpha_{t+1}{\langle \nabla f(\by_t),{\bz}_t\rangle}{\bz}_t$
  \STATE \textbf{end for}
\end{algorithmic}
 \caption{Accelerated Random Pursuit}
 \label{algo:RAMP}
\end{algorithm}

\begin{algorithm}[h]
\begin{algorithmic}[1]
  \STATE \textbf{init} $\bx_0 = \bv_0 = \by_0$, $\beta_0 = 0$, and $\nu$
  \STATE \textbf{for} {$t=0, 1 \dots T$}
     \STATE \qquad  Solve ${\alpha_{t+1}^2}{L\nu} = \beta_t + \alpha_{t+1}$
    \STATE  \qquad $\beta_{t+1} := \beta_t + \alpha_{t+1} $
    \STATE \qquad $\tau_t := \frac{\alpha_{t+1}}{\beta_{t+1}}$
   \STATE \qquad Compute $\by_{t} := (1- \tau_t)\bx_{t} + \tau_t \bv_t$
   \STATE \qquad Find $\bz_t := \lmo_\cA(\nabla f(\by_{t}))$
\STATE \qquad $\bx_{t+1} := \by_t - \frac{\langle \nabla f(\by_t),\bz_t\rangle}{L\|\bz_t\|_2^2}\bz_t$
\STATE \qquad  Sample $\tilde{\bz}_t  \sim \cZ$
\STATE \qquad $\bv_{t+1} := \bv_t - \alpha_{t+1}{\langle \nabla f(\by_t),\tilde{\bz}_t\rangle}\tilde{\bz}_t$
  \STATE \textbf{end for}
\end{algorithmic}
 \caption{Accelerated Matching Pursuit}
 \label{algo:ACDM}
\end{algorithm}

\subsection{Analysis } \label{subsec:pursuit_acc_theory}

 Modeling explicitly the dependency on the structure of the set is crucial to accelerate MP. Indeed, acceleration works by defining two different quadratic subproblems, one upper bound given by smoothness, and one lower bound given by a model of the function. The constraints on the set of possible descent direction implicitly used in MP influence both these subproblems. While the smoothness quadratic upper bound contains information about $\cA$ in its definition ($\by = \bx + \gamma \bz$ and $\|\bz\|_\cA = 1$), the model of the function needs explicit modeling of $\cA$. This is particularly crucial when sampling a direction in the model update, which can be thought as a sort of exploration part of the algorithm.
In both the algorithms, the update of the parameter $\bv$ corresponds to optimizing the modeling function $\psi$ which can be given as :
\begin{align}
&~~~~~~~~~~~\psi_{t+1}(\bx)  =
 \psi_t(\bx) +   \notag\\   &\alpha_{t+1} \Big( f(\by_t) + \langle \tilde{\bz}_t^\top \nabla f(\by_t) , \tilde{\bz}_t^\top \bP(\bx - \by_t)  \rangle \Big) \,, \label{eq:model_acc_random_paper}
\end{align}
where
$\psi_0(\bx) = \frac{1}{2}\| \bx - \bx_0 \|_{\bP}^2$.
\begin{lemma}\label{lem:v-min-psi}
The update of $\bv$ in Algorithm~\ref{algo:RAMP} and~\ref{algo:ACDM} minimizes the model\vspace{-1mm}
  $$
    \bv_{t} \in \argmin_{\bx}\psi_t(\bx) \,.
  $$
\end{lemma}

We will be first discussing the theory for the \textit{greedy} accelerated method in detail. As evident from the algorithm~\ref{algo:ACDM}, another important constant which is required for both the analysis and to actually run the algorithm is $\nu$ for which:\vspace{-2mm}
   $$
  \nu  \leq \max_{\bd\in\lin(\cA)}
  \frac{\bbE\encase{ (\tilde{\bz}_t^\top \bd)^2 \norm{\tilde{\bz}_t}_\bP^2}\|{\bz(\bd)}\|_2^2}{(\bz(\bd)^\top \bd)^2} \,,
  $$
  where $\bz(\bd)$ is defined to be\vspace{-2mm}
  $$
  	\bz(\bd) = \lmo_\cA(- \bd)\,.
  	 $$

The quantity $\nu$ relates the geometry of the atom set with the sampling procedure in a similar way as $\hat\delta^2$ in Equation~\eqref{eq:delta-def} but instead of measuring how much worse a random update is when compared to a steepest update.

\begin{theorem} \label{thm:greedy_acc_pursuit}
Let $f$ be a convex function and $\cA$ be a symmetric compact set. Then the output of algorithm~\ref{algo:ACDM} for any $t\geq 1$ converges with the following rate:\vspace{-1mm}
$$
    \bbE[f(\bx_t)] - f(\bx^\star) \leq \frac{2L \nu}{t(t+1)}\norm{\bx^\star - \bx_0}^2_\bP \,.
  $$
\end{theorem}

\begin{proof}
We extend the proof technique of~\cite{Lee13,stich2013optimization} to allow for general atomic updates. The analysis can be found in Appendix~\ref{subsec:proof}
\end{proof}

Once we understand the convergence of the greedy approach, the analysis of accelerated random pursuit can be derived easily. Here, we state the rate of convergence for accelerated random pursuit:
\begin{theorem} \label{thm:random_pursuit_acc}
Let $f$ be a convex function and $\cA$ be a symmetric set. Then the output of the algorithm~\ref{algo:RAMP} for any $t\geq 1$ converges with the following rate:\vspace{-1mm}
$$\bbE [f(\bx_t)] - f(\bx^\star) \leq \frac{2L \nu'}{t(t+1)}\norm{\bx^\star - \bx_0}^2_\bP\,,$$ where\vspace{-2mm}
$$
\nu'  \leq \max_{\bd\in\lin(\cA)}
\frac{\bbE\encase{ ({\bz}_t^\top \bd)^2 \norm{{\bz}_t}_\bP^2}}{\bbE\encase{ ({\bz}_t^\top \bd)^2/ \norm{{\bz}_t}_2^2}}\,.
$$
\end{theorem}

\paragraph{Discussion on Greedy Accelerated Coordinate Descent.}
The convergence rate for greedy accelerated coordinate descent can directly be obtained from the rate from accelerated matching pursuit. Let the atom set $\cA$ consist of the standard basis vectors $\{\be_i, i \in [n]\}$ and $\cZ$ be a uniform distribution over this set.  Then algorithm \ref{algo:RAMP} reduces to the accelerated randomized coordinate method (ACDM) of \cite{Lee13,Nesterov:2017} and we recover their rates. Instead if we use algorithm \ref{algo:ACDM}, we obtain a novel accelerated greedy coordinate method with a (potentially) better convergence rate.\footnote{Simultaneously (and independently) \cite{lu2018greedy} derived the same accelerated greedy coordinate algorithm.} 
\begin{lemma}\label{lem:coordinate_descent_acc}
	When $\cA = \{\be_i, i \in [n]\}$ and $\cZ$ is a uniform distribution over $\cA$, then $\bP = n \bI$, $\nu' = n$ and $\nu \in [1, n]$.
\end{lemma}
%\vspace{-4mm}
\section{Empirical Evaluation} \label{sec:expts}
In this section we aim at empirically validate our theoretical findings. In both experiments we use 1 and the intrinsic dimensionality of $\lin(\cA)$ as $\nu$ and $\nu'$ respectively. Note that a value of $\nu$ smaller than $\nu'$ represents the best case for the steepest update. We implicitly assume that the worst case in which a random update is as good as the steepest one never happens.

\textbf{Toy Data:} First, we report the function value while minimizing the squared distance between the a random 100 dimensional signal with both positive and negative entries and its sparse representation in terms of atoms. We sample a random dictionary containing 200 atoms which we then make symmetric. The result is depicted in Figure~\ref{fig:synth_data}. As anticipated from our analysis, the accelerated schemes converge much faster than the non-accelerated variants. Furthermore, in both cases the steepest update converge faster than the random one, due to a better dependency on the dimensionality of the space. 

\textbf{Real Data:}
We use the under-sampled Urban HDI Dataset from which we extract the dictionary of atoms using the hierarchical clustering approached of~\cite{gillis2015hierarchical}. This dataset contains 5'929 pixels, each associated with 162 hyperspectral features. The number of dictionary elements is 6, motivated by the fact that 6 different physical materials are depicted in this HSI data \cite{gillis2018fast}. We approximate each pixel with a linear combination of the dictionary elements by minimizing the square distance between the observed pixel and our approximation. We report in Figure~\ref{fig:real_data} the loss as an average across all the pixels:\vspace{-1mm}
\begin{align*}
\min_{\bx_i\in\lin(\cA)} \frac{1}{N}\sum_{i=1}^N \| \bx_i - \bb_i\|^2
\end{align*}
\begin{figure}
\centering
\begin{minipage}{0.48\textwidth}
\includegraphics[width=8cm,height=5cm]{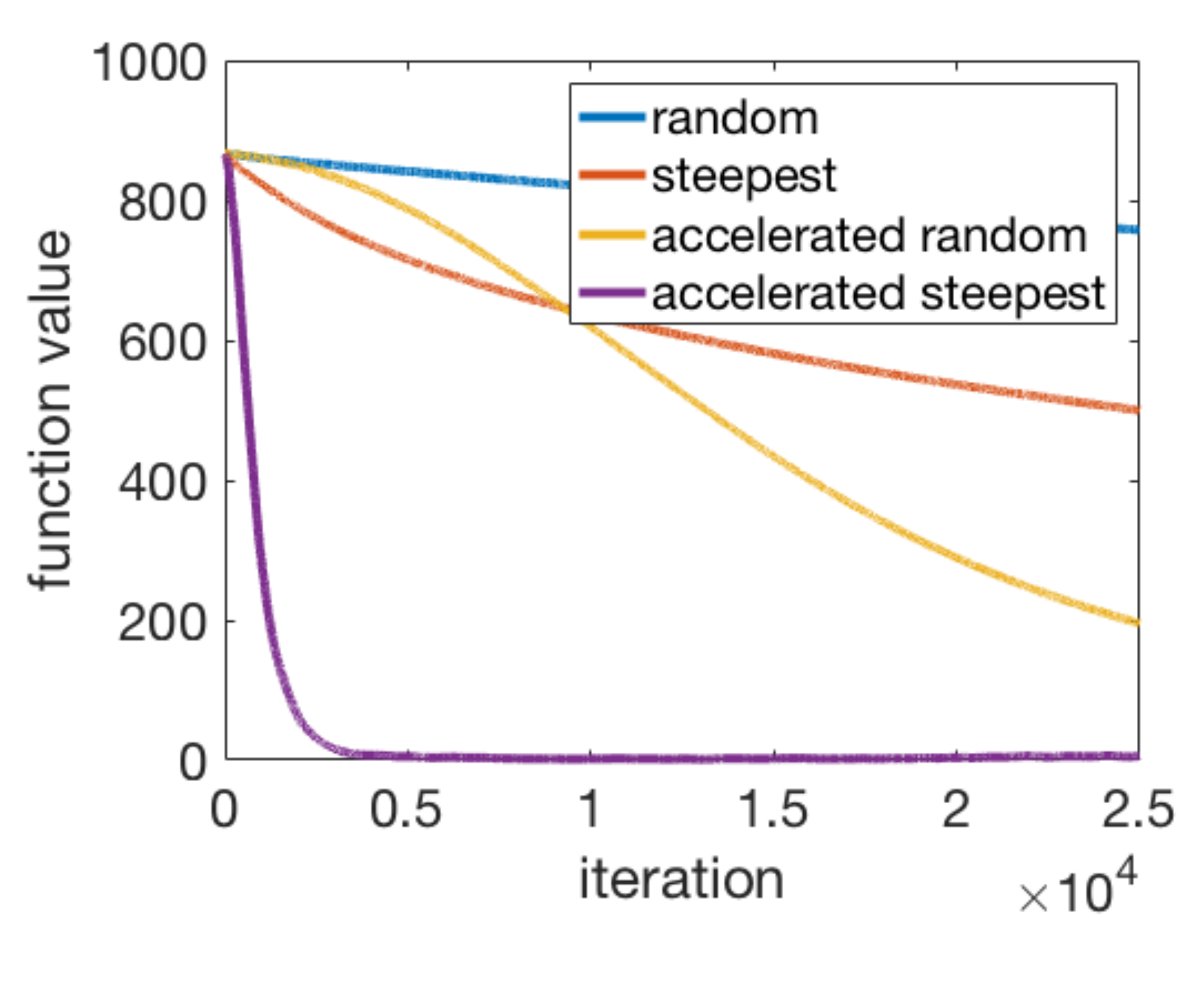}
\vspace{-3mm}
\caption{loss for synthetic data}\label{fig:synth_data}
\end{minipage}
\begin{minipage}{0.48\textwidth}
\includegraphics[width=8cm,height=5cm]{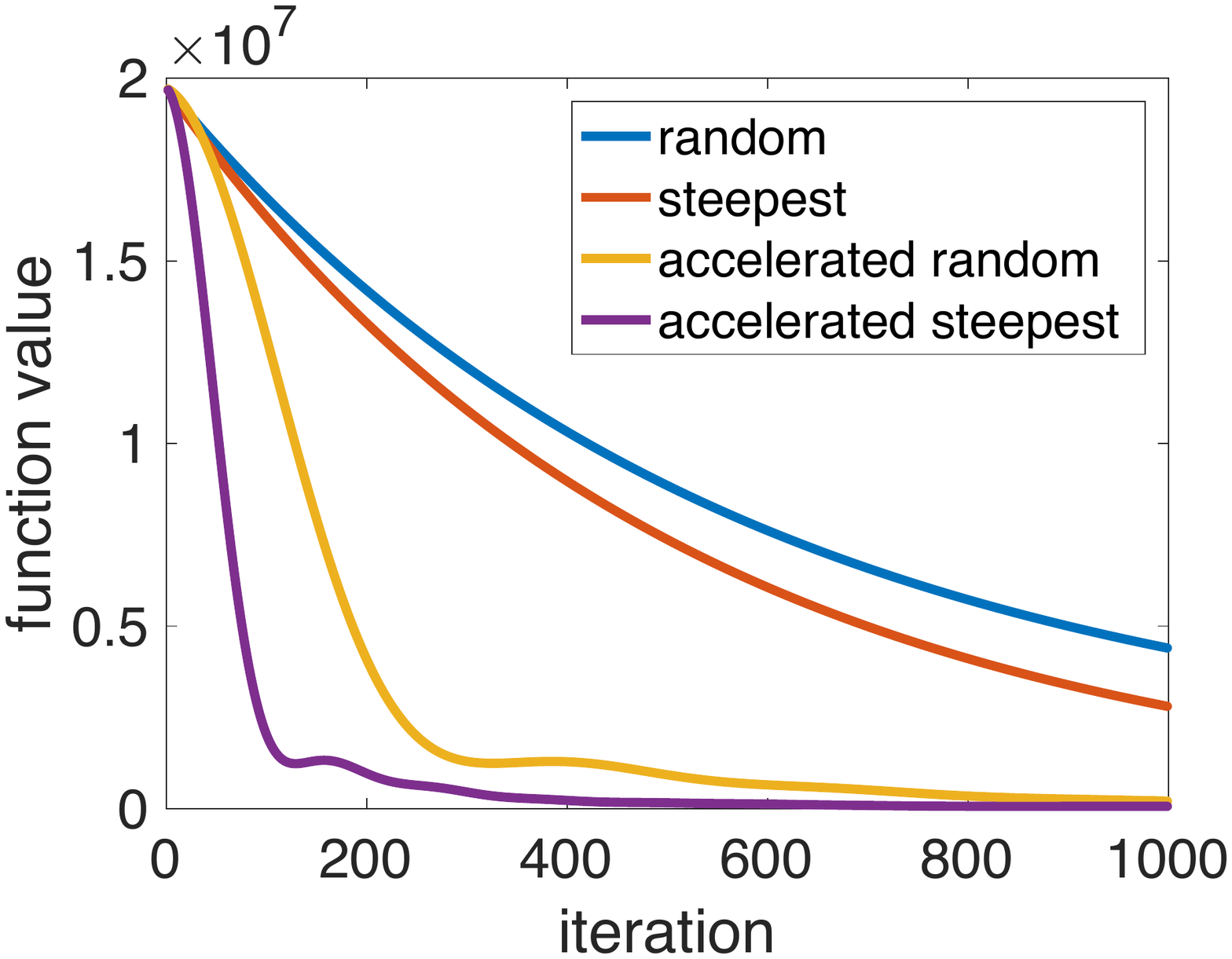}
\vspace{-3mm}
\caption{loss for  hyperspectral data}\label{fig:real_data}
\end{minipage}
\end{figure}

We notice that as expected, the steepest matching pursuit converges faster than the random pursuit, but as expected both of them converge at the same regime. On the other hand, the accelerated scheme converge much faster than the non-accelerated variants. Note that the acceleration kicks in only after a few iterations as the accelerated rate has a worse dependency on the intrinsic dimensionality of the linear span than the non accelerated algorithms.
We notice that the speedup of steepest MP is much more evident in the synthetic data. The reason is that this experiment is much more high dimensional than the hyperspectral data. Indeed, the span of the dictionary is a 6 dimensional manifold in the latter and the full ambient space in the former and the steepest update yields a better dependency on the dimensionality.
\vspace{-4mm}
\section{Conclusions}
In this paper we presented a unified analysis of matching pursuit and coordinate descent algorithms. As a consequence, we exploit the similarity between the two to obtain the best of both worlds: tight sublinear and linear rates for steepest coordinate descent and the first accelerated rate for matching pursuit and steepest coordinate descent. Furthermore, we discussed the relation between the steepest and the random directions by viewing the latter as an approximate version of the former. An affine invariant accelerated proof remains an open problem.

\section*{Acknowledgements} 
FL is supported by Max Planck/ETH Center for Learning Systems and partially supported by ETH core funding (to GR). SPK and MJ acknowledge support by SNSF project 175796, and SUS by the Microsoft Research Swiss JRC.
\bibliography{bibliography}
\bibliographystyle{icml2018}
\clearpage
\newpage
\onecolumn
\appendix
\begin{center}
{\centering \LARGE Appendix }
\vspace{1cm}
\sloppy
% !TEX root = arxiv_revisiting_mp.tex

\section{Sublinear Rates}
\begin{reptheorem}{lem:LwithRadius}
Assume $f$ is $L$-smooth w.r.t. a given norm $\|\cdot\|$, over $\lin(\cA)$ where $\cA$ is symmetric.
Then,
\begin{equation}
 L_\cA \leq L \, \radius_{\norm{\cdot}}(\cA)^2 \,.
\end{equation}
\begin{proof}
Let $D(\by,\bx):= f(\by)- f(\bx) + \gamma \langle \nabla f(\bx), \by - \bx \rangle$
By the definition of smoothness of $f$ w.r.t. $\|\cdot\|$,
\begin{eqnarray*}
D(\by,\bx) \leq \frac{L}{2} \| \by - \bx\|^2 \,.
\end{eqnarray*}

Hence, from the definition of $L_\cA$,
\begin{align*}
L_\cA &\leq \sup_{\substack{\bx,\by\in\lin(\cA)\\\by = \bx + \gamma\bz \\\|\bz\|_\cA=1, \gamma\in \bbR_{>0}}}
 \frac{2}{\gamma^2} \frac{L}{2}  \| \by - \bx\|^2 \\
&= L \sup_{\bz \ s.t. \|\bz\|_\cA = 1} \, \| \bs\|^2 \\
&= L \, \radius_{\norm{\cdot}}(\cA)^2 \ . \qedhere
\end{align*}
\end{proof}
\end{reptheorem}

The definition of the smoothness constant w.r.t. the atomic norm yields the following quadratic upper bound:
\begin{equation}
 L_{\cA} = \sup_{\substack{\bx,\by\in\lin(\cA)\\\by = \bx + \gamma\bz \\\|\bz\|_\cA=1, \gamma\in \bbR_{>0}}}\frac{2}{\gamma^2}\left[ f(\by)- f(\bx) +  \langle \nabla f(\bx), \by - \bx \rangle\right] \,.
\end{equation}
Furthermore, let:
\begin{equation}
R_\cA^2 = \max_{\substack{\bx\in\lin(\cA)\\ f(\bx)\leq f(\bx_0)}}\|\bx-\bx^\star\|^2_\cA \,.
\end{equation}
Now, we show that the algorithm we presented is affine invariant.
An optimization method is called \emph{affine invariant} if it is invariant
under affine transformations of the input problem: If one chooses any
re-parameterization of the domain~$\domain$ by a \emph{surjective} linear or
affine map $\bM:\hat\domain\rightarrow\domain$, then the ``old'' and ``new''
optimization problems $\min_{\bx\in\domain}f(\bx)$ and
$\min_{\hat\bx\in\hat\domain}\hat f(\hat\bx)$ for $\hat f(\hat\bx):=f(\bM\hat\bx)$
look the same to the algorithm. Note that $\nabla \hat f = \bM^T \nabla f$.

First of all, let us note that $L_\cA$ is affine invariant as it does not depend on any norm.
Now:
\begin{align*}
\bM\hat\bx_{t+1} &= \bM\left( \hat\bx_{t} + \frac{\langle \nabla \hat f(\hat\bx_t), \hat\bz_t\rangle}{L_\cA}\hat\bz_t\right)\\
&= \bM\hat\bx_{t} + \frac{\langle \nabla \hat f(\hat\bx_t), \hat\bz_t\rangle}{L_\cA}\bM\hat\bz_t\\
&= \bx_{t} + \frac{\langle \nabla \hat f(\hat\bx_t), \hat\bz_t\rangle}{L_\cA}\bz_t\\
&= \bx_{t} + \frac{\langle \bM^T\nabla f(\bx_t), \hat\bz_t\rangle}{L_\cA}\bz_t\\
&= \bx_{t} + \frac{\langle \nabla f(\bx_t), \bM\hat\bz_t\rangle}{L_\cA}\bz_t\\
&= \bx_{t} + \frac{\langle \nabla f(\bx_t), \bz_t\rangle}{L_\cA}\bz_t\\
&= \bx_{t+1} \,.
\end{align*}
Therefore the algorithm is affine invariant.

\subsection{Affine Invariant Sublinear Rate}
\begin{reptheorem}{thm:sublinear_MP_rate}
Let $\cA \subset \cH$ be a closed and bounded set. We assume that $\|\cdot\|_\cA$ is a norm over $\lin(\cA)$. Let $f$ be convex and $L_\cA$-smooth w.r.t. the norm $\|\cdot\|_\cA$ over $\lin(\cA)$, and let $R_\cA$ be the radius of the level set of $\bx_0$ measured with the atomic norm.
Then, Algorithm~\ref{algo:affineMP} converges for $t \geq 0$ as
\[
f(\bx_{t+1}) - f(\bx^\star) \leq \frac{2L_{\cA}R_\cA^2}{\delta^2(t+2)} \,,
\]
where $\delta \in (0,1]$ is the relative accuracy parameter of the employed approximate \lmo~\eqref{eq:inexactLMOMP}.
\end{reptheorem}

\begin{proof}

Recall that $\tilde{\bz}_t$ is the atom selected in iteration $t$ by the approximate \lmo  defined in \eqref{eq:inexactLMOMP}. We start by upper-bounding $f$  using the definition of $L_\cA$ as follows:
\begin{eqnarray}
f(\bx_{t+1}) &\leq &  \min_{\gamma\in\bbR} f(\bx_t) + \gamma \langle \nabla f(\bx_t),
 \tilde{\bz}_t \rangle   + \frac{\gamma^2}{2} L_{\cA}\|\bz\|^2_\cA \nonumber \\
& =& \min_{\gamma\in\bbR} f(\bx_t) + \gamma \langle \nabla f(\bx_t),
 \tilde{\bz}_t \rangle   + \frac{\gamma^2}{2} L_{\cA} \nonumber \\
& \leq &  f(\bx_t) - \frac{\langle\nabla f(\bx_t),\tilde{\bz}_t\rangle^2}{2L_{\cA}} \nonumber\\
& = &  f(\bx_t) - \frac{\langle\nabla_\parallel f(\bx_t),\tilde{\bz}_t\rangle^2}{2L_{\cA}}\nonumber \\
& \leq &  f(\bx_t) - \delta^2\frac{\langle\nabla_\parallel f(\bx_t),\bz_t\rangle^2}{2L_{\cA}} \,. \nonumber
\end{eqnarray}
Where $\nabla_\parallel f$ is the parallel component of the gradient wrt the linear span of $\cA$. Note that $\|\bd\|_{\cA*}:= \sup\left\lbrace \langle\bz,\bd\rangle, \bz\in\cA\right\rbrace$ is the dual of the atomic norm.
Therefore, by definition:
\begin{align*}
\langle\nabla_\parallel f(\bx_t),\bz_t\rangle^2 = \|-\nabla_\parallel f(\bx_t)\|^2_{\cA*} \,,
\end{align*}
which gives:
\begin{align*}
f(\bx_{t+1}) &\leq f(\bx_t) - \delta^2\frac{1}{2L_{\cA}} \|\nabla_\parallel f(\bx_t)\|^2_{\cA*}\\
&\leq f(\bx_t) - \delta^2\frac{1}{2L_{\cA}}\frac{ \left(-\langle\nabla_\parallel f(\bx_t),\bx_t-\bx^\star\rangle\right)^2}{R_\cA^2}\\
&= f(\bx_t) - \delta^2\frac{1}{2L_{\cA}}\frac{ \left(\langle\nabla_\parallel f(\bx_t),\bx_t-\bx^\star\rangle\right)^2}{R_\cA^2}\\
&\leq f(\bx_t) - \delta^2\frac{1}{2L_{\cA}}\frac{ \left(f(\bx_t)-f(\bx^\star)\right)^2}{R_\cA^2} \,,
\end{align*}
where the second inequality is Cauchy-Schwarz and the third one is convexity.
Which gives:
\begin{align*}
\epsilon_{t+1} &\leq \frac{2L_{\cA}R_\cA^2}{\delta^2(t+2)} \,. \qedhere
\end{align*}
\end{proof}

\subsection{Randomized Affine Invariant Sublinear Rate}

For random sampling of $\bz$ from a distribution over $\cA$, let
\begin{equation}
\hat\delta^2 := \min_{\bd\in\lin\cA}\frac{\bbE_{\bz \in \cA} \langle \bd,\bz\rangle^2}{\|\bd\|_{\cA*}^2} \,.
\end{equation}
\begin{reptheorem}{thm:rand_purs_sub}
Let $\cA \subset \cH$ be a closed and bounded set. We assume that $\|\cdot\|_\cA$ is a norm. Let $f$ be convex and $L_\cA$ smooth w.r.t. the norm $\|\cdot\|_\cA$ over $\lin(\cA)$ and let $R_\cA$ be the radius of the level set of $\bx_0$ measured with the atomic norm.
Then, Algorithm~\ref{algo:affineMP} converges for $t \geq 0$ as
\[
\bbE_\bz \big[f(\bx_{t+1}) \big] - f(\bx^\star) \leq \frac{2L_{\cA}R_\cA^2}{\hat\delta^2(t+2)} \,,
\]
when the $\lmo$ is replaced with random sampling of $\bz$ from a distribution over $\cA$.
\end{reptheorem}

\begin{proof}

Recall that $\tilde{\bz}_t$ is the atom selected in iteration $t$ by the approximate \lmo  defined in \eqref{eq:inexactLMOMP}. We start by upper-bounding $f$  using the definition of $L_\cA$ as follows
\begin{eqnarray}
\bbE_\bz f(\bx_{t+1}) &\leq & \bbE_\bz\left[ \min_{\gamma\in\bbR} f(\bx_t) + \gamma \langle \nabla f(\bx_t),
 \bz \rangle   + \frac{\gamma^2}{2} L_{\cA}\|\bz\|^2_\cA\right] \nonumber \\
& =& \bbE_\bz\left[\min_{\gamma\in\bbR} f(\bx_t) + \gamma \langle \nabla f(\bx_t),
 \bz \rangle   + \frac{\gamma^2}{2} L_{\cA}\right] \nonumber \\
& \leq &  f(\bx_t) - \frac{\bbE_\bz\left[\langle\nabla f(\bx_t),\bz\rangle^2\right]}{2L_{\cA}} \nonumber\\
& = &  f(\bx_t) - \frac{\bbE_\bz\left[\langle\nabla_\parallel f(\bx_t),\bz\rangle^2\right]}{2L_{\cA}}\nonumber \\
& \leq &  f(\bx_t) - \hat\delta^2\frac{\langle\nabla_\parallel f(\bx_t),\bz_t\rangle^2}{2L_{\cA}} \,.\nonumber 
\end{eqnarray}
The rest of the proof proceeds as in Theorem~\ref{thm:sublinear_MP_rate}.
\end{proof}

\section{Linear Rates}
\subsection{Affine Invariant Linear Rate}
Let us first the fine the affine invariant notion of strong convexity based on the atomic norm:
\begin{align*}
\mu_\cA := \inf_{\substack{\bx,\by\in\lin\cA\\\bx\neq\by}} \frac{2}{\|\by-\bx\|_{\cA}^2}\left[f(\by)-f(\bx)-\langle \nabla f(\bx),\by -\bx\rangle\right] \,.
\end{align*}
Let us recall the definition of \textit{minimal directional width} from~\cite{locatello2017unified}:
\begin{align*}
\mdw := \min_{\substack{\bd\in\lin(\cA)\\\bd\neq 0}} \max_{z\in\cA}\langle \frac{\bd}{\|\bd\|},\bz\rangle \,.
\end{align*}
Then, we can relate our new definition of strong convexity with the $\mdw$ as follows.
\begin{reptheorem}{thm:mu_mdw}
Assume $f$ is $\mu$ strongly convex wrt a given norm $\|\cdot\|$ over $\lin(\cA)$ and $\cA$ is symmetric. Then:
\begin{align*}
\mu_\cA \geq \mdw^2\mu \,.
\end{align*}
\end{reptheorem}

\begin{proof}
First of all, note that for any $\bx,\by\in\lin(\cA)$ with $\bx\neq\by$ we have that:
\begin{align*}
\langle\nabla f(x),x-y\rangle^2\leq \|\nabla f(\bx)\|_{\cA*}^2\|\bx - \by\|_{\cA}^2 \,.
\end{align*}
Therefore:
\begin{align*}
\mu_\cA &= \inf_{\substack{\bx,\by\in\lin\cA\\\bx\neq\by}}\frac{2}{\|\by-\bx\|_{\cA}^2} D(\bx,\by)\\
&\geq \inf_{\substack{\bx,\by,\bd\in\lin\cA\\\bx\neq\by,\bd\neq 0}} \frac{\|\bd\|_{\cA*}^2}{\langle\bd,\bx-\by\rangle^2}2D(\bx,\by)\\
&\geq\inf_{\substack{\bx,\by,\bd\in\lin\cA\\\bx\neq\by,\bd\neq 0}}\frac{\|\bd\|_{\cA*}^2}{\langle\bd,\bx-\by\rangle^2}\mu \|\bx-\by\|^2\\
&\geq\inf_{\substack{\bx,\by,\bd\in\lin\cA\\\bx\neq\by,\bd\neq 0}}\frac{\|\bd\|_{\cA*}^2}{\langle\bd,\frac{\bx-\by}{\|\bx-\by\|}\rangle^2}\mu\\
&\geq\inf_{\substack{\bx,\by,\bd\in\lin\cA\\\bx\neq\by,\bd\neq 0}}\frac{\|\bd\|_{\cA*}^2}{\|\bd\|^2}\mu\\
&\geq\inf_{\substack{\bd\in\lin\cA\\\bd\neq 0}} \max_{z}\frac{\langle \bd,\bz\rangle^2}{\|\bd\|^2}\mu\\
&=\mdw^2\mu \,. \qedhere
\end{align*}
\end{proof}
\begin{reptheorem}{thm:linear_rate}
(Part 1). 
Let $\cA \subset \cH$ be a closed and bounded set. We assume that $\|\cdot\|_\cA$ is a norm. Let $f$ be $\mu_\cA$-strongly convex and $L_\cA$-smooth w.r.t. the norm $\|\cdot\|_\cA$, both over $\lin(\cA)$.
Then, Algorithm~\ref{algo:affineMP} converges for $t \geq 0$ as
\[
\epsilon_{t+1} \leq \big(1 - \delta^2\frac{\mu_\cA}{L_{\cA}}\big)\epsilon_t \,.
\]
where $\epsilon_{t} := f(\bx_{t}) - f(\bx^\star)$.
\end{reptheorem}

\begin{proof}

Recall that $\tilde{\bz}_t$ is the atom selected in iteration $t$ by the approximate \lmo  defined in \eqref{eq:inexactLMOMP}. We start by upper-bounding $f$  using the definition of $L_\cA$ as follows
\begin{eqnarray}
f(\bx_{t+1}) &\leq &  \min_{\gamma\in\bbR} f(\bx_t) + \gamma \langle \nabla f(\bx_t),
 \tilde{\bz}_t \rangle   + \frac{\gamma^2}{2} L_{\cA}\|\bz\|^2_\cA \nonumber \\
& =& \min_{\gamma\in\bbR} f(\bx_t) + \gamma \langle \nabla f(\bx_t),
 \tilde{\bz}_t \rangle   + \frac{\gamma^2}{2} L_{\cA} \nonumber \\
& \leq &  f(\bx_t) - \frac{\langle\nabla f(\bx_t),\tilde{\bz}_t\rangle^2}{2L_{\cA}} \nonumber\\
& = &  f(\bx_t) - \frac{\langle\nabla_\parallel f(\bx_t),\tilde{\bz}_t\rangle^2}{2L_{\cA}}\nonumber \\
& \leq &  f(\bx_t) - \delta^2\frac{\langle\nabla_\parallel f(\bx_t),\bz_t\rangle^2}{2L_{\cA}}\nonumber\\
& = &  f(\bx_t) - \delta^2\frac{\langle-\nabla_\parallel f(\bx_t),\bz_t\rangle^2}{2L_{\cA}} \,. \nonumber 
\end{eqnarray}
Where $\|\bd\|_{\cA*}:= \sup\left\lbrace \langle\bz,\bd\rangle, \bz\in\cA\right\rbrace$ is the dual of the atomic norm.
Therefore, by definition:
\begin{align*}
\langle-\nabla_\parallel f(\bx_t),\bz_t\rangle^2 = \|\nabla_\parallel f(\bx_t)\|^2_{\cA*} \,,
\end{align*}
which gives:
\begin{align*}
f(\bx_{t+1}) &\leq f(\bx_t) - \delta^2\frac{1}{2L_{\cA}} \|\nabla_\parallel f(\bx_t)\|^2_{\cA*} \,.
\end{align*}
From strong convexity we have that:
\begin{align*}
f(\by)\geq f(\bx) + \langle\nabla f(\bx),\by-\bx\rangle + \frac{\mu_\cA}{2}\|\by-\bx\|_\cA^2 \,.
\end{align*}
Fixing $\by = \bx_t + \gamma(\bx^\star - \bx_t)$ and $\gamma = 1$ in the LHS and minimizing the RHS we obtain:
\begin{align*}
f(\bx^\star)&\geq f(\bx_t) - \frac{1}{2\mu_\cA}\frac{\langle\nabla f(\bx_t),\bx^\star-\bx_t\rangle}{\|\bx^\star-\bx_t\|_\cA^2}\\
&\geq f(\bx_t) - \frac{1}{2\mu_\cA}\|\nabla_\parallel f(\bx_t)\|_{\cA^*}^2 \,,
\end{align*}
where the last inequality is obtained by the fact that $\langle\nabla f(\bx_t),\bx^\star-\bx_t\rangle=\langle\nabla_{\parallel} f(\bx_t),\bx^\star-\bx_t\rangle$and  Cauchy-Schwartz.
Therefore:
\begin{align*}
\|\nabla f(\bx_t)\|_{\cA^*}\geq 2\epsilon_t\mu_\cA \,,
\end{align*}
which yields:
\begin{align*}
\epsilon_{t+1} &\leq\epsilon_t - \delta^2\frac{\mu_\cA}{L_{\cA}}\epsilon_t\,. \qedhere
\end{align*}

\end{proof}

\subsection{Randomized Affine Invariant Linear Rate}
\begin{reptheorem}{thm:linear_rate}
(Part 2). 
Let $\cA \subset \cH$ be a closed and bounded set. We assume that $\|\cdot\|_\cA$ is a norm. Let $f$ be $\mu_\cA$-strongly convex and $L_\cA$-smooth w.r.t. the norm $\|\cdot\|_\cA$, both over $\lin(\cA)$.
Then, Algorithm~\ref{algo:affineMP} converges for $t \geq 0$ as
\[
\bbE_\bz \left[\epsilon_{t+1}|\bx_t\right] \leq \big(1 - \hat\delta^2\frac{\mu_\cA}{L_{\cA}} \big)\epsilon_t \,,
\]
where $\epsilon_{t} := f(\bx_{t}) - f(\bx^\star)$, and the LMO direction $\bz$ is sampled randomly from $\cA$, from the same distribution as used in the definition of $\hat\delta$.
\end{reptheorem}

\begin{proof}

We start by upper-bounding $f$  using the definition of $L_\cA$ as follows
\begin{eqnarray}
\bbE_\bz \left[f(\bx_{t+1})\right] &\leq &  \bbE_\bz\left[ \min_{\gamma\in\bbR} f(\bx_t) + \gamma \langle \nabla f(\bx_t),
 \tilde{\bz}_t \rangle   + \frac{\gamma^2}{2} L_{\cA}\|\bz\|^2_\cA\right] \nonumber \\
& =& \bbE_\bz\left[\min_{\gamma\in\bbR} f(\bx_t) + \gamma \langle \nabla f(\bx_t),
 \tilde{\bz}_t \rangle   + \frac{\gamma^2}{2} L_{\cA} \right]\nonumber \\
& \leq &  f(\bx_t) - \bbE_\bz\left[\frac{\langle\nabla f(\bx_t),\tilde{\bz}_t\rangle^2}{2L_{\cA}}\right] \nonumber\\
& \leq &  f(\bx_t) - \hat\delta^2\frac{\langle\nabla f(\bx_t),\bz_t\rangle^2}{2L_{\cA}}\nonumber \\
& = &  f(\bx_t) - \hat\delta^2\frac{\langle\nabla_\parallel f(\bx_t),\tilde{\bz}_t\rangle^2}{2L_{\cA}}\nonumber \\
& = &  f(\bx_t) - \hat\delta^2\frac{\langle-\nabla_\parallel f(\bx_t),\tilde{\bz}_t\rangle^2}{2L_{\cA}} \,. \nonumber
\end{eqnarray}
The rest of the proof proceeds as in Part 1 of the proof of Theorem~\ref{thm:linear_rate}.
\end{proof}

\section{Accelerated Matching Pursuit}\label{app:acc_match_pursuit}
Our proof follows the technique for acceleration given in~\cite{Lee13,Nesterov:2017,Nesterov:2004gx,stich2013optimization} 

\begin{algorithm}[h]
\begin{algorithmic}[1]
  \STATE \textbf{init} $\bx_0 = \bv_0 = \by_0$, $\beta_0 = 0$, and $\nu$
  \STATE \textbf{for} {$t=0, 1 \dots T$}
     \STATE \qquad  Solve ${\alpha_{t+1}^2}{L\nu} = \beta_t + \alpha_{t+1}$
    \STATE  \qquad $\beta_{t+1} = \beta_t + \alpha_{t+1} $
    \STATE \qquad $\tau_t = \frac{\alpha_{t+1}}{\beta_{t+1}}$
   \STATE \qquad Compute $\by_{t} = (1- \tau_t)\bx_{t} + \tau_t \bv_t$
   \STATE \qquad Find $\bz_t := \lmo_\cA(\nabla f(\by_{t}))$
\STATE \qquad $\bx_{t+1} = \by_t - \frac{\langle \nabla f(\by_t),\bz_t\rangle}{L\|\bz_t\|_2^2}\bz_t$
\STATE \qquad  sample $\tilde{\bz}_t ~~ \sim \cZ$
\STATE \qquad $\bv_{t+1} = \bv_t - \alpha_{t+1}{\langle \nabla f(\by_t),\tilde{\bz}_t\rangle}\tilde{\bz}_t$
  \STATE \textbf{end for}
\end{algorithmic}
 \caption{Accelerated Matching Pursuit}
 \label{algo:ACDM_Appendix}
\end{algorithm}

\subsection{Proof of Convergence} \label{subsec:proof}
We define $\norm{\bx}^2_\bP = \bx^\top \bP\bx $. We start our proof by first defining the model function $\psi_t$. For $t=0$, we define :
$$\psi_{0}(\bx) = \frac{1}{2}\norm{\bx - \bv_0}_\bP^2 \,.$$
Then for $t >1$, $\psi_t$ is inductively defined as
\begin{align}
\psi_{t+1}(\bx) &= \psi_t(\bx) +  \alpha_{t+1} \Big( f(\by_t) + \langle \tilde{\bz}_t^\top \nabla f(\by_t) , \tilde{\bz}_t^\top \bP(\bx - \by_t)  \rangle \Big) \,. \label{eq:model_acc_random}
\end{align}

\begin{proof}[\textbf{Proof of Lemma~\ref{lem:v-min-psi}}]
  We will prove the statement inductively. For $t=0$, $\psi_{0}(\bx) = \frac{1}{2}\norm{\bx - \bv_0}_\bP^2$ and so the statement holds. Suppose it holds for some $t \geq 0$. Observe that the function $\psi_t(\bx)$ is a quadratic with Hessian $\bP$. This means that we can reformulate $\psi_t(\bx)$ with minima at $\bv_t$ as
  $$
    \psi_t(\bx) = \psi_t(\bv_t) + \frac{1}{2}\norm{\bx - \bv_t}_\bP^2 \,.
  $$
  Using this reformulation,
  \begin{align*}
    \argmin_{\bx}\psi_{t+1}(\bx) &= \argmin_{\bx} \Big\{ \psi_t(\bx) +  \alpha_{t+1} \Big( f(\by_t) + \langle \tilde{\bz}_t^\top \nabla f(\by_t) , \tilde{\bz}_t^\top \bP(\bx - \by_t)  \rangle \Big)\Big\}\\
    &= \argmin_{\bx} \Big\{ \psi_t(\bv_t) + \frac{1}{2}\norm{\bx - \bv_t}_\bP^2 +  \alpha_{t+1} \Big( f(\by_t) + \langle \tilde{\bz}_t^\top \nabla f(\by_t) , \tilde{\bz}_t^\top \bP(\bx - \by_t)  \rangle \Big)\Big\}\\
    &= \argmin_{\bx} \Big\{ \frac{1}{2}\norm{\bx - \bv_t}_\bP^2 +  \alpha_{t+1}  \langle \tilde{\bz}_t^\top \nabla f(\by_t) , \tilde{\bz}_t^\top \bP(\bx - \bv_t)  \rangle \Big\}\\
    &= \bv_t - \alpha_{t+1}{\langle \nabla f(\by_t),\tilde{\bz}_t\rangle}\tilde{\bz}_t\\
    &= \bv_{t+1}\,. \qedhere
  \end{align*}
\end{proof}
\begin{lemma}[Upper bound on $\psi_t(\bx)$]\label{lem:psi-upper}
  $$
    \bbE[\psi_t(\bx)] \leq \beta_t f(\bx) + \psi_0(\bx) \,.
  $$
\end{lemma}
\begin{proof}
  We will also show this through induction. The statement is trivially true for $t=0$ since $\beta_0 = 0$. Assuming the statement holds for some $t \geq 0$,
  \begin{align*}
    \bbE[\psi_{t+1}(\bx)] &= \bbE\Big[\psi_t(\bx) +  \alpha_{t+1} \Big( f(\by_t) + \langle \tilde{\bz}_t^\top \nabla f(\by_t) , \tilde{\bz}_t^\top \bP(\bx - \by_t)  \rangle \Big)\Big]\\
    &= \bbE\Big[\psi_t(\bx)\Big] +  \alpha_{t+1}\bbE\Big[ \Big( f(\by_t) + \langle \tilde{\bz}_t^\top \nabla f(\by_t) , \tilde{\bz}_t^\top \bP(\bx - \by_t)  \rangle \Big)\Big] \\
    &\leq \beta_{t}f(\bx) + \psi_0(\bx) +  \alpha_{t+1} \Big( f(\by_t) + \nabla f(\by_t)^\top \bbE\Big[ \tilde{\bz}_t \tilde{\bz}_t^\top\Big] \bP(\bx - \by_t)  \rangle \Big)\\
    &= \beta_{t}f(\bx) + \psi_0(\bx) +  \alpha_{t+1} \Big( f(\by_t) + \nabla f(\by_t)^\top \bP^{-1} \bP (\bx - \by_t)  \rangle \Big)\\
    &= \beta_{t}f(\bx) + \psi_0(\bx) +  \alpha_{t+1} \Big( f(\by_t) + \nabla f(\by_t)^\top (\bx - \by_t)  \rangle \Big)\\
    &\leq  \beta_{t}f(\bx) + \psi_0(\bx) +  \alpha_{t+1} f(\bx)\,.
  \end{align*}
  In the above, we used the convexity of the function $f(\bx)$ and the definition of $\bP$.
\end{proof}
\begin{lemma}[Bound on progress]\label{lem:progress-bound}
	For any $t\geq 0$ of algorithm \ref{algo:ACDM_Appendix},
	$$
		f(\bx_{t+1}) - f(\by_t) \leq - \frac{1}{2L \norm{\bz_t}^2_2}\nabla f(\by_t)^\top\encase{ \bz_t\bz_t^\top} \nabla f(\by_t)\,.
	$$
\end{lemma}
\begin{proof}
	  The update $\bx_{t+1}$ along with the smoothness of $f(\bx)$ guarantees that for $\gamma_{t+1} = \frac{\langle \nabla f(\by_t),\bz_t\rangle}{L\|\bz_t\|^2}$,
	\begin{align*}
	f(\bx_{t+1}) &= f(\by_t + \gamma_{t+1} \bz_t) \\
	&\leq f(\by_t) + \gamma_{t+1}\inner{\nabla f(\by_t)}{\bz_t} + \frac{L\gamma_{t+1}^2}{2}\norm{\bz_t}^2 \\
	&= f(\by_t) - \frac{1}{2L \norm{\bz_t}^2_2}\nabla f(\by_t)^\top\encase{ \bz_t\bz_t^\top} \nabla f(\by_t)\,.
	\end{align*}
\end{proof}
\begin{lemma}[Lower bound on $\psi_t(\bx)$]\label{lem:psi-lower}
Given a filtration $\cF_t$ upto time step $t$, $$
  \bbE[\min_{\bx}\psi_t(\bx)| \cF_t] \geq \beta_t f(\bx_t) \,.
$$
\end{lemma}
\begin{proof}
  This too we will show inductively. For $t=0$, $\psi_t(\bx) = \frac{1}{2}\norm{\bx - \bv_0}_\bP^2 \geq 0$ with $\beta_0 =0$. Assume the statement holds for some $t \geq 0$. Recall that $\psi_t(\bx)$ has a minima at $\bv_t$ and can be alternatively formulated as $\psi_t(\bv_t) + \frac{1}{2}\norm{\bx - \bv_t}_\bP^2$. Using this,
  \begin{align*}
    \psi_{t+1}^\star &= \min_{\bx}\encase{\psi_t(\bx) + \alpha_{t+1}\brac{\inner{{\tilde{\bz}}_{t}^\top\nabla f(\by_t)}{{\tilde{\bz}}_{t}^\top \bP(\bx - \by_t)} + f(\by_t)}}\\
    &= \min_{\bx}\encase{\psi_t(\bv_t) + \alpha_{t+1}\brac{\inner{{\tilde{\bz}}_{t}^\top\nabla f(\by_t)}{{\tilde{\bz}}_{t}^\top \bP(\bx - \by_t)} + \frac{1}{2 \alpha_{t+1}}\norm{\bx - \bv_t}_\bP^2 + f(\by_t)}}\\
    &= \psi_t^\star + \alpha_{t+1} f(\by_t) + \alpha_{t+1} \min_{\bx}\encase{\inner{\bP {\tilde{\bz}}_{t}{\tilde{\bz}}_{t}^\top\nabla f(\by_t)}{\bx - \by_t} + \frac{1}{2 \alpha_{t+1}}\norm{\bx - \bv_t}_\bP^2 }\,.
  \end{align*}
  Since we defined $\by_t = (1- \tau_t)\bx_t + \tau_t \bv_t$, rearranging the terms gives us that
  $$
    \by_t - \bv_t = \frac{1 - \tau_t}{\tau_t}(\bx_t - \by_t)\,.
  $$
  Let us take now compute $\bbE[\psi_{t+1}^\star|\cF_t]$ by combining the above two equations:
  \begin{align*}
    \bbE[\psi_{t+1}^\star|\cF_t] &= \psi_t^\star + \alpha_{t+1} f(\by_t) + \frac{\alpha_{t+1} (1 - \tau_t)}{\tau_t} \inner{\bP \bbE_t[\tilde{\bz}_t\tilde{\bz}_t^\top]\nabla f(\by_t)}{\by_t - \bx_t} \\ &\hspace{1in} + \alpha_{t+1}\bbE_t \min_{\bx}\encase{\inner{\bP \tilde{\bz}_t\tilde{\bz}_t^\top\nabla f(\by_t)}{\bx - \bv_t} + \frac{1}{2 \alpha_{t+1}}\norm{\bx - \bv_t}_\bP^2 }\\
    &= \psi_t^\star + \alpha_{t+1} f(\by_t) + \frac{\alpha_{t+1} (1 - \tau_t)}{\tau_t} \inner{\nabla f(\by_t)}{\by_t - \bx_t} \\ &\hspace{1in} + \alpha_{t+1} \bbE_t \min_{\bx}\encase{\inner{\bP \tilde{\bz}_t\tilde{\bz}_t^\top\nabla f(\by_t)}{\bx - \bv_t} + \frac{1}{2 \alpha_{t+1}}\norm{\bx - \bv_t}_\bP^2 }\\
    &= \psi_t^\star + \alpha_{t+1} f(\by_t) + \frac{\alpha_{t+1} (1 - \tau_t)}{\tau_t} \inner{\nabla f(\by_t)}{\by_t - \bx_t} \\ &\hspace{1in} - \frac{\alpha_{t+1}^2}{2} \nabla f(\by_t)^\top \bbE_t\encase{ \tilde{\bz}_t\tilde{\bz}_t^\top \bP \bP^{-1} \bP \tilde{\bz}_t\tilde{\bz}_t^\top}\nabla f(\by_t)\\
    &= \psi_t^\star + \alpha_{t+1} f(\by_t) + \frac{\alpha_{t+1} (1 - \tau_t)}{\tau_t} \inner{\nabla f(\by_t)}{\by_t - \bx_t} \\ &\hspace{1in} - \frac{\alpha_{t+1}^2}{2} \nabla f(\by_t)^\top \bbE_t\encase{ \tilde{\bz}_t\tilde{\bz}_t^\top \bP \tilde{\bz}_t\tilde{\bz}_t^\top}\nabla f(\by_t)\,.
  \end{align*}
  
  Let us define a constant $\nu \geq 0$ such that it is the smallest number for which the below inequality holds for all $t$,
  $$
  \nu \nabla f(\by_t)^\top\frac{\encase{ \bz_t\bz_t^\top}}{2L \norm{\bz_t}^2_2} \nabla f(\by_t) \geq
  \nabla f(\by_t)^\top \bbE\encase{ \tilde{\bz}_t\tilde{\bz}_t^\top \bP \tilde{\bz}_t\tilde{\bz}_t^\top}\nabla f(\by_t)\,.
  $$

  Also recall from Lemma \ref{lem:progress-bound} that
  $$
  f(\bx_{t+1}) - f(\by_t) \leq - \frac{1}{2L \norm{\bz_t}^2_2}\nabla f(\by_t)^\top\encase{ \bz_t\bz_t^\top} \nabla f(\by_t)\,.
  $$
  Using the above two statements in our computation of $\psi_{t+1}^\star$, we get
  \begin{align*}
  \bbE[\psi_{t+1}^\star | \cF_t] &= \psi_t^\star + \alpha_{t+1} f(\by_t) + \frac{\alpha_{t+1} (1 - \tau_t)}{\tau_t} \inner{\nabla f(\by_t)}{\by_t - \bx_t} \\ &\hspace{1in} - \frac{\alpha_{t+1}^2}{2} \nabla f(\by_t)^\top \bbE_t\encase{ \tilde{\bz}_t\tilde{\bz}_t^\top \bP \tilde{\bz}_t\tilde{\bz}_t^\top}\nabla f(\by_t) \\
  &\geq \psi_t^\star + \alpha_{t+1} f(\by_t) + \frac{\alpha_{t+1} (1 - \tau_t)}{\tau_t} \inner{\nabla f(\by_t)}{\by_t - \bx_t} \\ &\hspace{1in} - \frac{\alpha_{t+1}^2 \nu}{2} \nabla f(\by_t)^\top\encase{ \bz_t\bz_t^\top} \nabla f(\by_t)\\
  &\geq \psi_t^\star + \alpha_{t+1} f(\by_t) + \frac{\alpha_{t+1} (1 - \tau_t)}{\tau_t} \inner{\nabla f(\by_t)}{\by_t - \bx_t} \\ &\hspace{1in} + {\alpha_{t+1}^2 L\nu} (f(\bx_{t+1}) - f(\by_t))\\
  &\geq \psi_t^\star + \alpha_{t+1} f(\by_t) + \frac{\alpha_{t+1} (1 - \tau_t)}{\tau_t} (f(\by_t) - f(\bx_t)) \\ &\hspace{1in} + {\alpha_{t+1}^2 L\nu} (f(\bx_{t+1}) - f(\by_t))\,.
  \end{align*}

  Let us pick $\alpha_{t+1}$ such that it satisfies ${\alpha_{t+1}^2}{\nu L} = \beta_{t+1}$. Then the above equation simplifies to
  \begin{align*}
  \bbE[\psi_{t+1}^\star| \cF_t] &\geq \psi_t^\star + \frac{\alpha_{t+1}}{\tau_t} f(\by_t) - \frac{\alpha_{t+1} (1 - \tau_t)}{\tau_t} f(\bx_t) + \beta_{t+1} (f(\bx_{t+1}) - f(\by_t))\\
  &= \psi_t^\star - \beta_t f(\bx_t) + \beta_{t+1}f(\by_t) - \beta_{t+1}f(\by_t) + \beta_{t+1} f(\bx_{t+1}) \\
  &= \psi_t^\star - \beta_t f(\bx_t) + \beta_{t+1} f(\bx_{t+1}) \,.
  \end{align*}

  We used that $\tau_t = \alpha_{t+1}/\beta_{t+1}$. Finally we use the inductive hypothesis to conclude that
  \begin{align*}
  \bbE[\psi_{t+1}^\star| \cF_t] &\geq \psi_t^\star - \beta_t f(\bx_t) + \beta_{t+1} f(\bx_{t+1}) \geq \beta_{t+1} f(\bx_{t+1})\,. \qedhere
  \end{align*}
\end{proof}

\begin{lemma}[Final convergence rate]\label{lem:acc-mp-rate}
	For any $t \geq 1$ the output of algorithm \ref{algo:ACDM_Appendix} satisfies:
  $$
    \bbE[f(\bx_t)] - f(\bx^\star) \leq \frac{2L \nu}{t(t+1)}\norm{\bx^\star - \bx_0}^2_\bP \,.
  $$
\end{lemma}
\begin{proof}

Putting together Lemmas \ref{lem:psi-upper} and \ref{lem:psi-lower}, we have that
$$
  \beta_t \bbE[f(\bx_t)] \leq \bbE[\psi_t^\star] \leq \bbE[\psi_t(\bx^\star)] \leq \beta_t f(\bx^\star) + \psi_0(\bx^\star)\,.
$$
Rearranging the terms we get
$$
  \bbE[f(\bx_t)] - f(\bx^\star) \leq \frac{1}{2 \beta_t}\norm{\bx^\star - \bx_0}^2_\bP\,.
$$
To finish the proof of the theorem, we only have to compute the value of $\beta_t$.
Recall that
$$
  {\alpha_{t+1}^2}{L\nu} = \beta_t + \alpha_{t+1} \,.
$$
We will inductively show that $\alpha_t \geq \frac{t}{2L \nu}$. For $t=0$, $\beta_0 =0$ and $\alpha_1 = \frac{1}{2L \nu}$ which satisfies the condition. Suppose that for some $t\geq 0$, the inequality holds for all iterations $i \leq t$. Recall that $\beta_t = \sum_{i=1}^t \alpha_i$ i.e. $\beta_t \geq \frac{t(t+1)}{4 L \nu}$. Then
\begin{align*}
  (\alpha_{t+1}L\nu)^2 - \alpha_{t+1}L \nu = \beta_t L \nu
  \geq \frac{t(t+1)}{4}\,.
\end{align*}
The positive root of the quadratic $x^2 - x - c = 0$ for $c \geq 0$ is $x = \frac{1}{2}\brac{1 + \sqrt{4c + 1}}$. Thus
$$
  \alpha_{t+1}L \nu \geq \frac{1}{2}\brac{1 + \sqrt{t(t+1) + 1}} \geq \frac{t + 1}{2}\,.
$$
This finishes our induction and proves the final rate of convergence.
\end{proof}

\begin{lemma}[Understanding $\nu$]
  $$
    \nu \leq \max_{\bd\in\lin(\cA)}
  \frac{\bbE\encase{ (\tilde{\bz}_t^\top \bd)^2 \norm{\tilde{\bz}_t}_\bP^2}\|{\bz(\bd)}\|_2^2}{(\bz(\bd)^\top \bd)^2} \,,
  $$
\end{lemma}
\begin{proof}
  Recall the definition of $\nu$ as a constant which satisfies the following inequality for all iterations $t$
  $$
  \nu \nabla f(\by_t)^\top\frac{\encase{ \bz_t\bz_t^\top}}{2L \norm{\bz_t}^2_2} \nabla f(\by_t) \geq
  \nabla f(\by_t)^\top \bbE\encase{ \tilde{\bz}_t\tilde{\bz}_t^\top \bP \tilde{\bz}_t\tilde{\bz}_t^\top}\nabla f(\by_t)\,.
  $$
which then yields the following sufficient condition for $\nu$:
  $$
  \nu  \leq \max_{\bd\in\lin(\cA)}
  \frac{\bbE\encase{ (\tilde{\bz}_t^\top \bd)^2 \norm{\tilde{\bz}_t}_\bP^2}\|{\bz(\bd)}\|_2^2}{(\bz(\bd)^\top \bd)^2} \,,
  $$
  where $\bz(\bd)$ is defined to be\vspace{-2mm}
  $$
  \bz(\bd) =  \lmo_\cA(- \bd)\,.
  $$
\end{proof}
\paragraph{Proof of Theorem \ref{thm:random_pursuit_acc}.}
The proof of Theorem \ref{thm:random_pursuit_acc} is exactly the same as that of the previous except that now the update to $\bv_t$ is also a random variable. The only change needed is the definition of $\nu'$ where we need the following to hold:
$$
\nu' \nabla f(\by_t)^\top\frac{1}{2L}\bbE_t\encase{ \bz_t\bz_t^\top/{\norm{\bz_t}^2_2}} \nabla f(\by_t) \geq
\nabla f(\by_t)^\top \bbE\encase{ {\bz}_t{\bz}_t^\top \bP {\bz}_t{\bz}_t^\top}\nabla f(\by_t)\,.
$$ \qedhere

\paragraph{Proof of Lemma \ref{lem:coordinate_descent_acc}.}
When $\cA = \{\be_i, i \in [n]\}$ and $\cZ$ is a uniform distribution over $\cA$, then $\tilde{\bP} = 1/n \bI$ and $\bP = n\bI$. A simple computation shows that $\nu' = n$ and $\nu \in [1, n]$. Note that here $\nu$ could be upto $n$ times smaller than $\nu'$ meaning that our accelerated greedy coordinate descent algorithm could be $\sqrt{n}$ times faster than the accelerated random coordinate descent. In the worst case $\nu = \nu'$, but in practice one can pick a smaller $\nu$ compared to $\nu'$ as the worst case gradient rarely happen. It is possible to tune $\nu$ and $\nu'$ empirically but we do not explore this direction.
\end{center}

 \clearpage

\end{document}